\documentclass[twoside,11pt]{article}

\usepackage{etoolbox}
\newbool{arxiv}

\booltrue{arxiv} 

\notbool{arxiv}{
    \usepackage{ntheorem} 
    \usepackage{jmlr2e}
} {
    \usepackage{amsthm}
    \usepackage{url}
    \usepackage[authoryear]{natbib}
    \usepackage{amsfonts} 
    \usepackage{amssymb} 
    \newtheorem{theorem}{Theorem}
    \newtheorem{proposition}[theorem]{Proposition}
    \newtheorem{remark}{Remark}
    \newtheorem{example}{Example}
    \bibliographystyle{abbrvnat}

    \usepackage{geometry}
    \geometry{margin=1.2in}
}

\usepackage{amsmath}
\usepackage{mathtools}
\usepackage{bm}
\usepackage{mathrsfs}

\usepackage[hang]{footmisc}

\usepackage[capitalise,noabbrev]{cleveref}

\usepackage{algorithm}
\usepackage{algpseudocode}

\usepackage{accents}

\notbool{arxiv}{
    \hypersetup{hidelinks}
    \usepackage{lastpage}
    \jmlrheading{25}{2024}{1-\pageref{LastPage}}{7/23}{1/24}{23-1015}{Giordano and Ingram and Broderick}

    \firstpageno{1}

    \editor{Matthew Hoffman}
} {
}

\input{_math_macros}    
\usepackage[]{graphicx}
\usepackage[]{color}
\makeatletter
\def\maxwidth{ %
  \ifdim\Gin@nat@width>\linewidth
    \linewidth
  \else
    \Gin@nat@width
  \fi
}
\makeatother

\definecolor{fgcolor}{rgb}{0.345, 0.345, 0.345}

\usepackage{framed}
\makeatletter
 {\par\unskip\endMakeFramed%
 \at@end@of@kframe}
\makeatother

\definecolor{shadecolor}{rgb}{.97, .97, .97}
\definecolor{messagecolor}{rgb}{0, 0, 0}
\definecolor{warningcolor}{rgb}{1, 0, 1}
\definecolor{errorcolor}{rgb}{1, 0, 0}
\newenvironment{knitrout}{}{} 

\usepackage{alltt}

\newcommand{\ARMNumModels}{53}
\newcommand{\MCParamDim}{124}
\newcommand{\OccParamDim}{1,884}
\newcommand{\PotusParamDim}{15,098}
\newcommand{\TennisParamDim}{5,014}
\newcommand{\ARMMinParamDim}{2}
\newcommand{\ARMMedParamDim}{5}
\newcommand{\ARMMaxParamDim}{176}
\newcommand{\TennisNUTSMinutes}{57}
\newcommand{\PotusNUTSMinutes}{643}
\newcommand{\MCNUTSMinutes}{597}
\newcommand{\OccNUTSMinutes}{251}
\newcommand{\ARMMinNUTSSeconds}{15}
\newcommand{\ARMMedNUTSSeconds}{39}
\newcommand{\ARMMaxNUTSMinutes}{16}
\newcommand{\DADVINumDraws}{30}

\newcommand{\ArmModels}{separation (2), wells\_dist100 (2), nes2000\_vote (2), wells\_d100ars (3), earn\_height (3), sesame\_one\_pred\_b (3), radon\_complete\_pool (3), earnings1 (3), kidscore\_momiq (3), kidscore\_momhs (3), electric\_one\_pred (3), sesame\_one\_pred\_a (3), sesame\_one\_pred\_2b (3), logearn\_height (3), electric\_multi\_preds (4), congress (4), wells\_interaction\_c (4), earnings2 (4), logearn\_logheight (4), kidiq\_multi\_preds (4), wells\_dae (4), wells\_interaction (4), logearn\_height\_male (4), ideo\_reparam (5), logearn\_interaction (5), kidscore\_momwork (5), kidiq\_interaction (5), wells\_dae\_c (5), mesquite\_volume (5), earnings\_interactions (5), wells\_dae\_inter (5), wells\_daae\_c (6), wells\_dae\_inter\_c (7), mesquite\_vash (7), wells\_predicted\_log (7), mesquite (8), mesquite\_vas (8), mesquite\_log (8), sesame\_multi\_preds\_3b (9), sesame\_multi\_preds\_3a (9), pilots (17), election88 (53), radon\_intercept (88), radon\_no\_pool (89), radon\_group (90), electric (100), electric\_1b (101), electric\_1a (109), electric\_1c (114), hiv (170), hiv\_inter (171), radon\_vary\_si (174), radon\_inter\_vary (176)}
\newcommand{\TennisNumCGParams}{20}
\newcommand{\OccNumCGParams}{20}

\newcommand{\TracesARM}{

\begin{knitrout}
\definecolor{shadecolor}{rgb}{0.969, 0.969, 0.969}\color{fgcolor}\begin{figure}[!h]

{\centering \includegraphics[width=0.98\linewidth,height=0.653\linewidth]{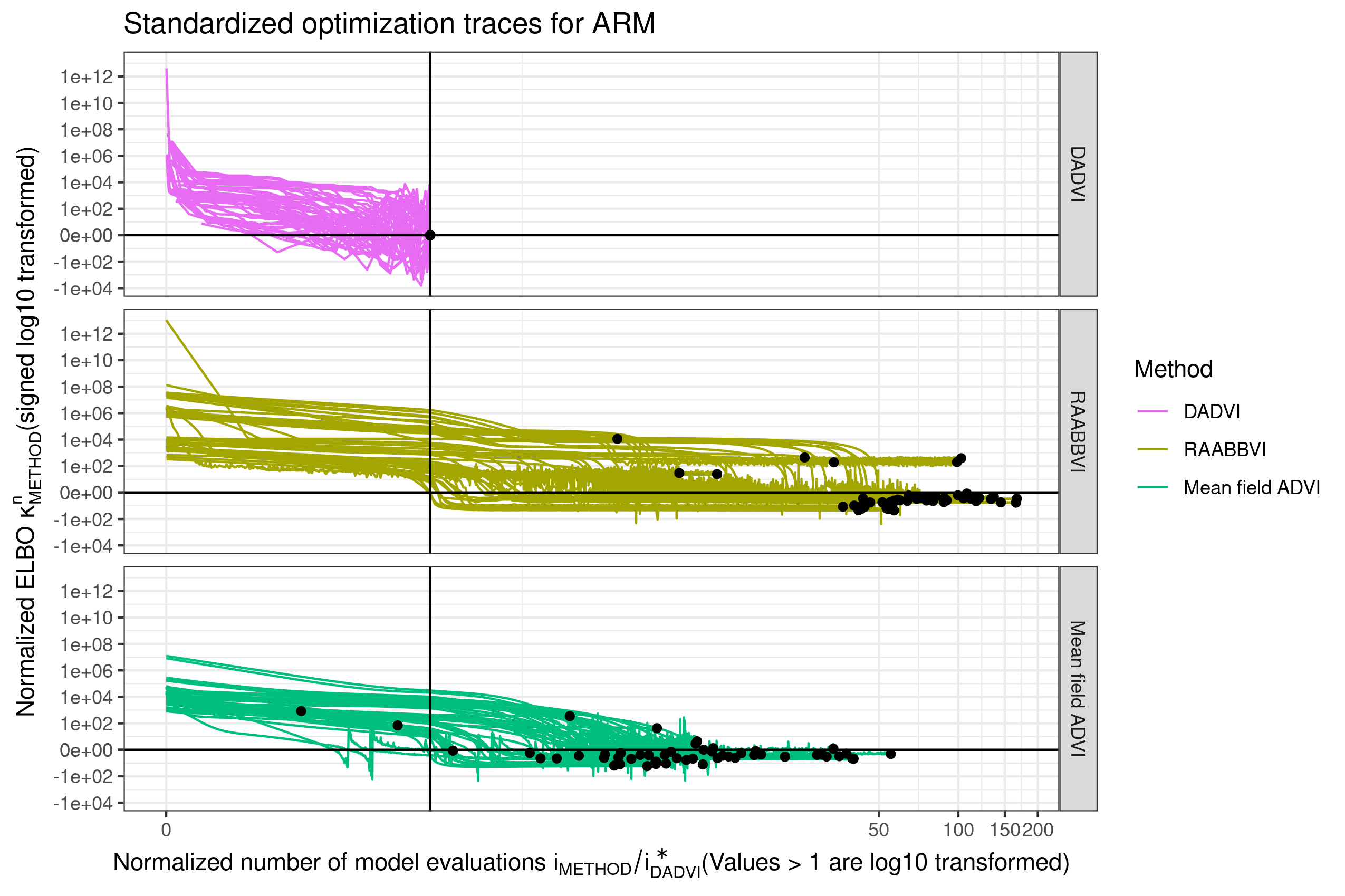} 

}

\caption[Optimization traces for the ARM models]{Optimization traces for the ARM models.  Black dots show the termination point of each method. Dots above the horizontal black line mean that DADVI found a better ELBO. Dots to the right of the vertical black line mean that DADVI terminated sooner in terms of model evaluations.}\label{fig:traces_arm_graph}
\end{figure}

\end{knitrout}
}

\newcommand{\TracesNonARM}{

\begin{knitrout}
\definecolor{shadecolor}{rgb}{0.969, 0.969, 0.969}\color{fgcolor}\begin{figure}[!h]

{\centering \includegraphics[width=0.98\linewidth,height=0.653\linewidth]{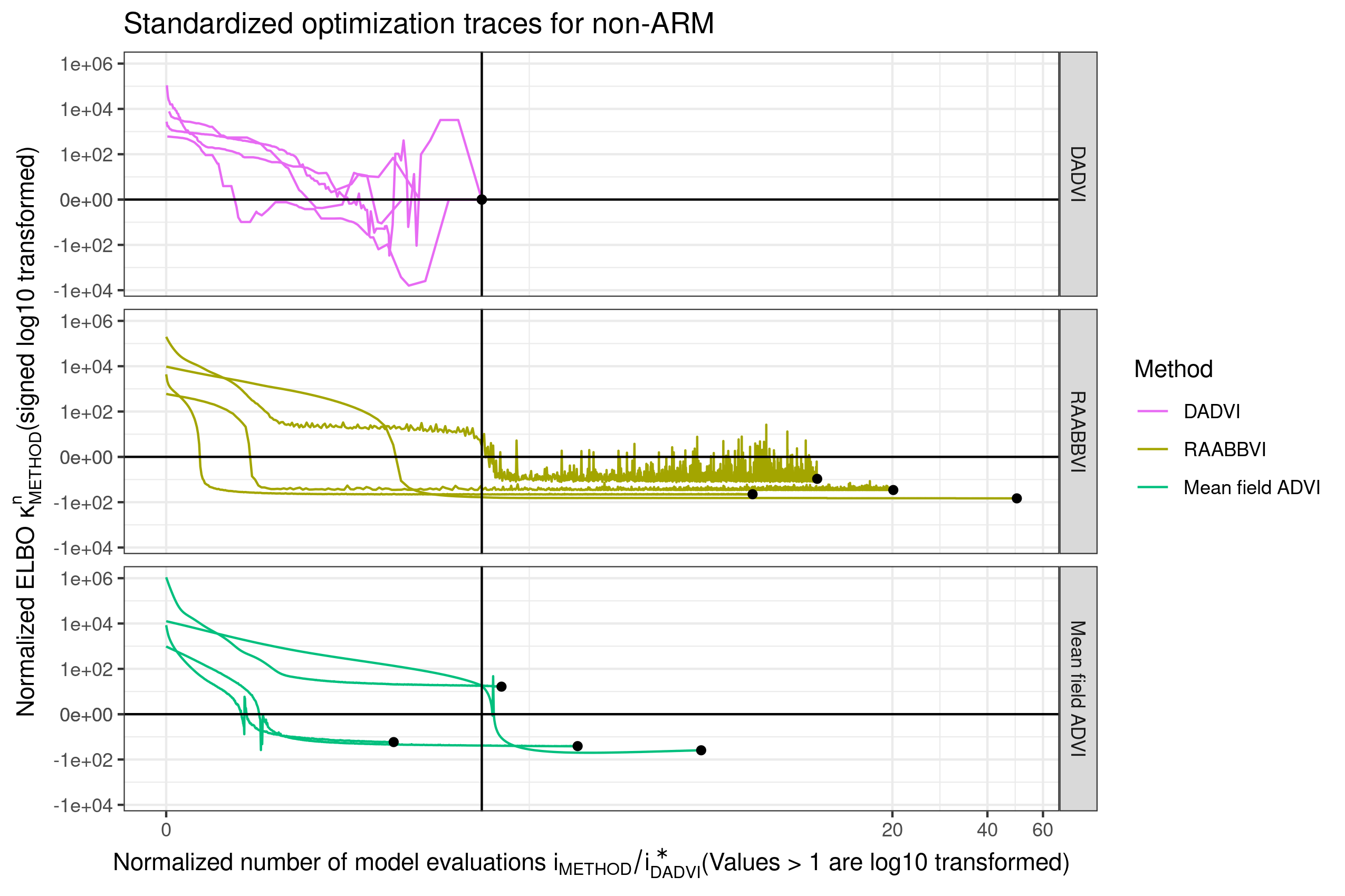} 

}

\caption[Traces for non-ARM models]{Traces for non-ARM models.  Black dots show the termination point of each method. Dots above the horizontal black line mean that DADVI found a better ELBO. Dots to the right of the vertical black line mean that DADVI terminated sooner in terms of model evaluations.}\label{fig:traces_nonarm_graph}
\end{figure}

\end{knitrout}
}

\newcommand{\RuntimeARM}{

\begin{knitrout}
\definecolor{shadecolor}{rgb}{0.969, 0.969, 0.969}\color{fgcolor}\begin{figure}[!h]

{\centering \includegraphics[width=0.98\linewidth,height=0.653\linewidth]{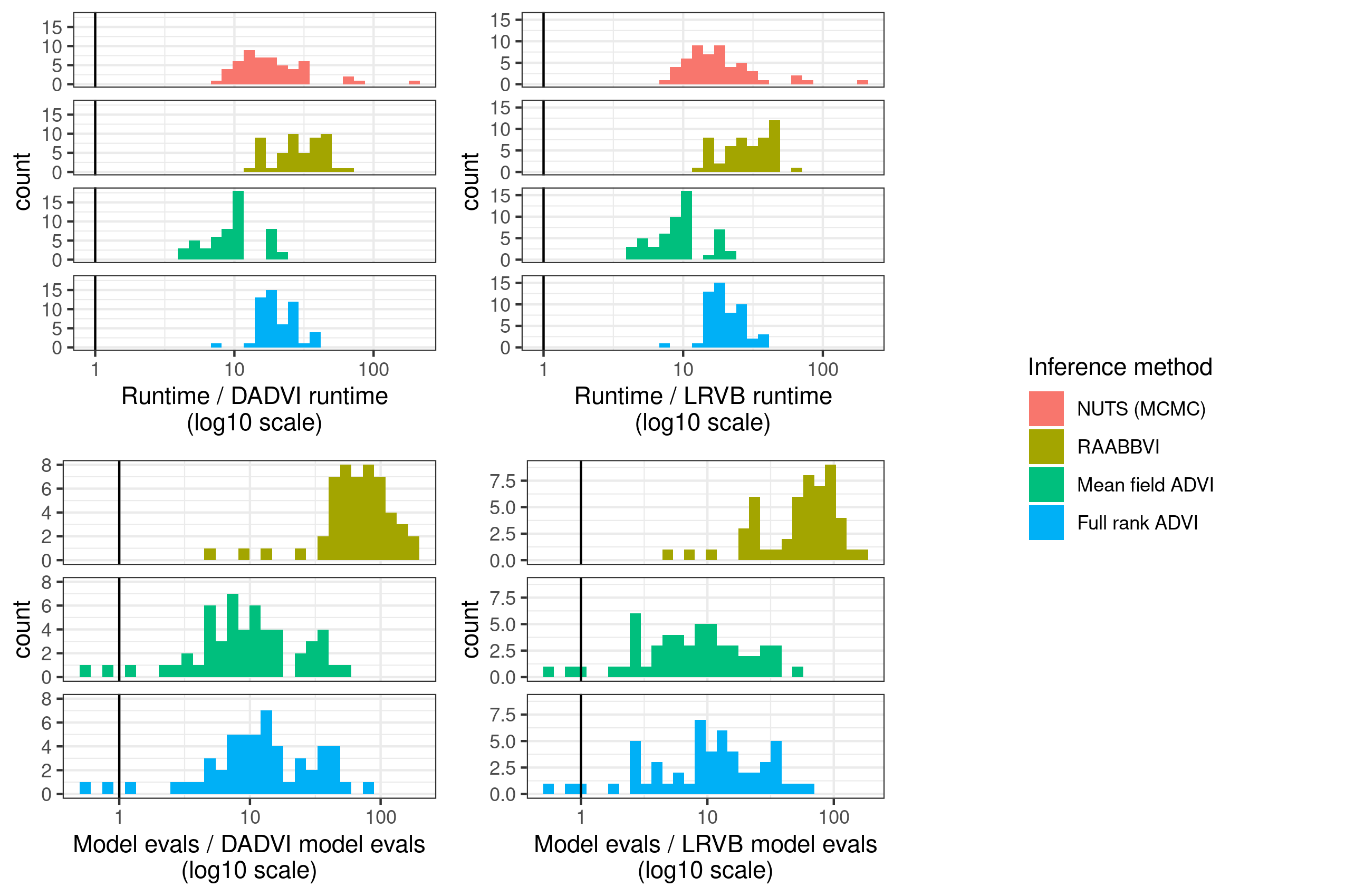} 

}

\caption[Runtimes and model evaluation counts for the ARM models]{Runtimes and model evaluation counts for the ARM models. Results are reported divided by the corresponding value for DADVI or LRVB.  Numbers greater than one (shown by the black line) indicate favorable performance by DADVI or LRVB.  Recall that the reported LRVB numbers include the cost of the DADVI optimization as well as the LR covariances.  Most of the ARM models are relatively low-dimensional, so the LR covariances added little to the computation.}\label{fig:runtimes_arm_graph}
\end{figure}

\end{knitrout}
}

\newcommand{\RuntimeNonARM}{

\begin{knitrout}
\definecolor{shadecolor}{rgb}{0.969, 0.969, 0.969}\color{fgcolor}\begin{figure}[!h]

{\centering \includegraphics[width=0.98\linewidth,height=0.653\linewidth]{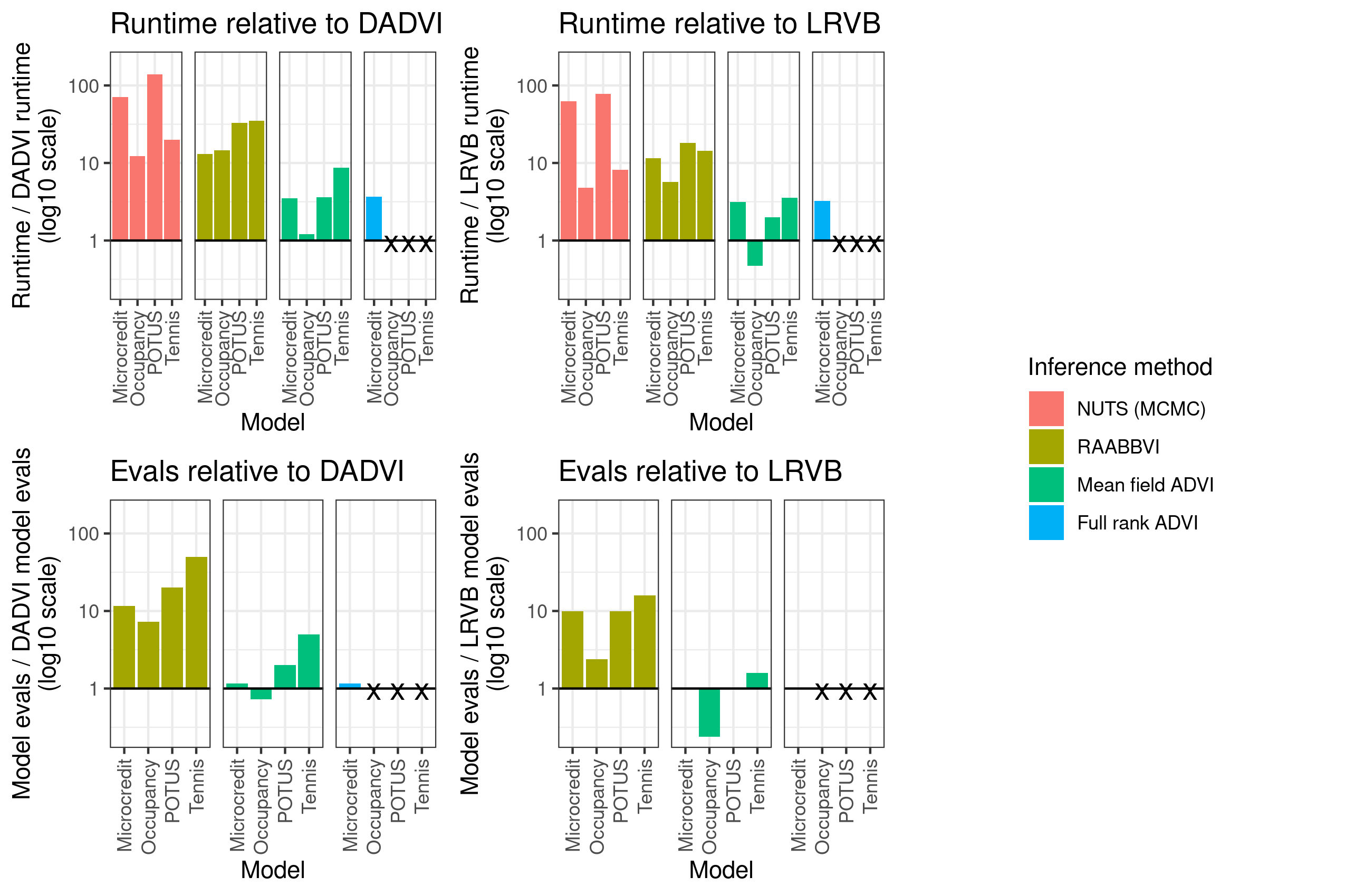} 

}

\caption[Runtimes and model evaluation counts for the non-ARM models]{Runtimes and model evaluation counts for the non-ARM models. Results are reported divided by the corresponding value for DADVI or LRVB.  Numbers greater than one (shown by the black line) indicate favorable performance by DADVI or LRVB. Recall that the reported LRVB numbers include the cost of the DADVI optimization as well as the LR covariances. Missing model and method combinations are marked with an X.}\label{fig:runtimes_nonarm_graph}
\end{figure}

\end{knitrout}
}

\newcommand{\PosteriorAccuracyARM}{

\begin{knitrout}
\definecolor{shadecolor}{rgb}{0.969, 0.969, 0.969}\color{fgcolor}\begin{figure}[!h]

{\centering \includegraphics[width=0.98\linewidth,height=0.653\linewidth]{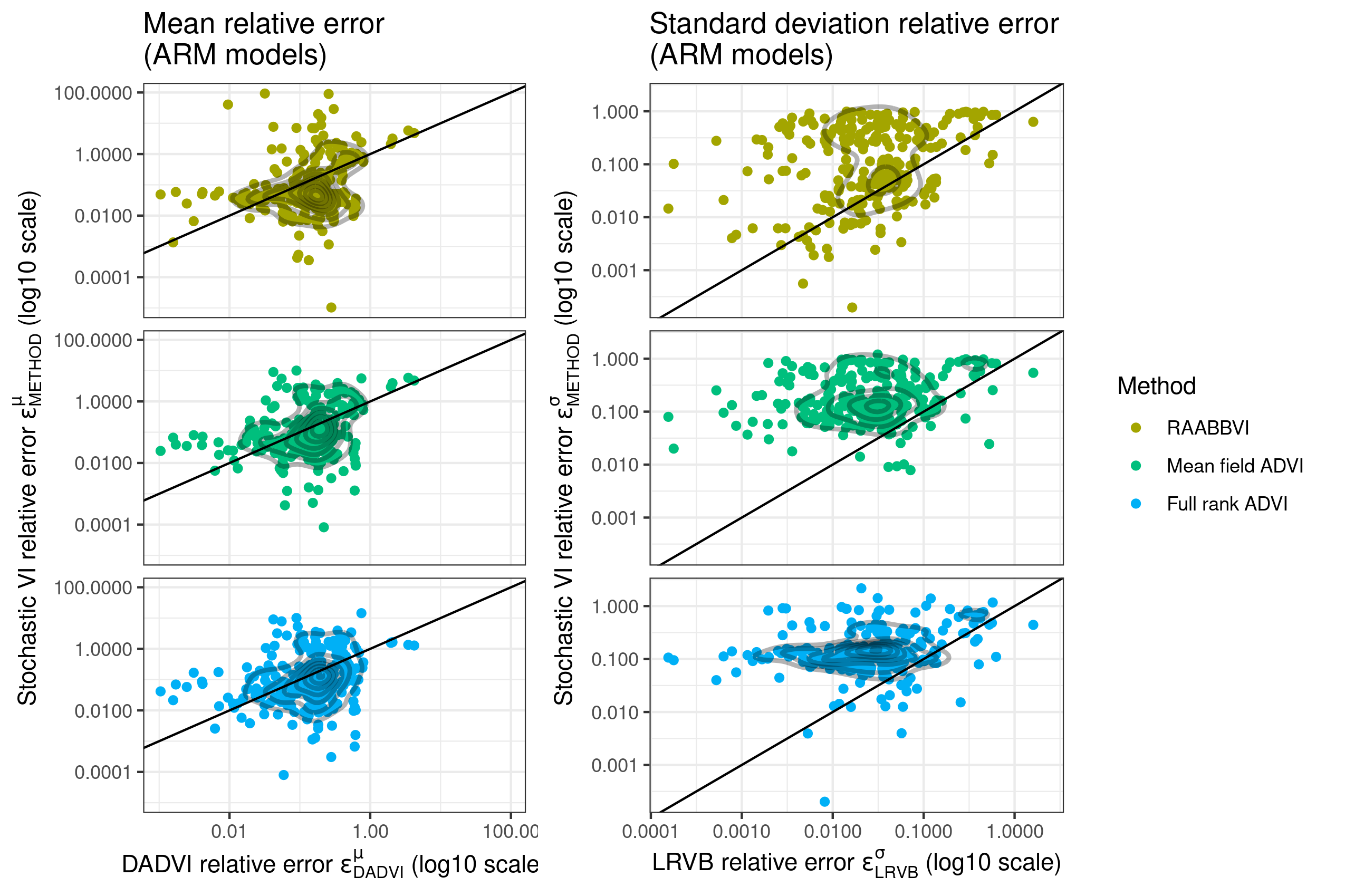} 

}

\caption[Posterior accuracy measures for the ARM models]{Posterior accuracy measures for the ARM models. Each point is a single named parameter in a single model. Points above the diagonal line indicate better DADVI or LRVB performance.  Level curves of a 2D density estimator are shown to help visualize overplotting.}\label{fig:posterior_arm_graph}
\end{figure}

\end{knitrout}
}

\newcommand{\PosteriorAccuracyNonARM}{

\begin{knitrout}
\definecolor{shadecolor}{rgb}{0.969, 0.969, 0.969}\color{fgcolor}\begin{figure}[!h]

{\centering \includegraphics[width=0.98\linewidth,height=0.653\linewidth]{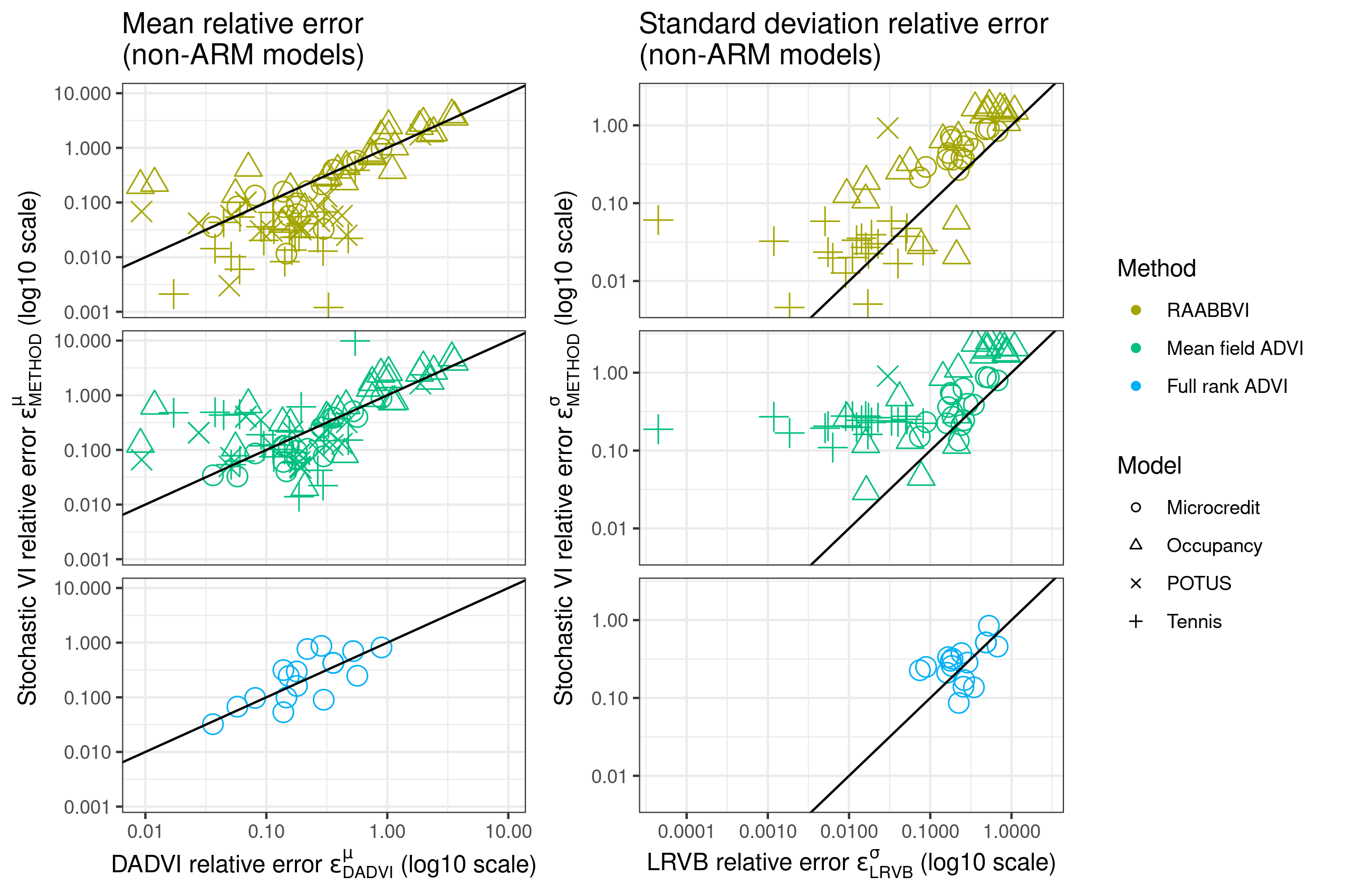} 

}

\caption[Posterior accuracy measures for the non-ARM models]{Posterior accuracy measures for the non-ARM models. Each point is a single named parameter in a single model. Points above the diagonal line indicate better DADVI or LRVB performance. }\label{fig:posterior_nonarm_graph}
\end{figure}

\end{knitrout}
}

\newcommand{\CoverageHistogram}{

\begin{knitrout}
\definecolor{shadecolor}{rgb}{0.969, 0.969, 0.969}\color{fgcolor}\begin{figure}[!h]

{\centering \includegraphics[width=0.98\linewidth,height=0.653\linewidth]{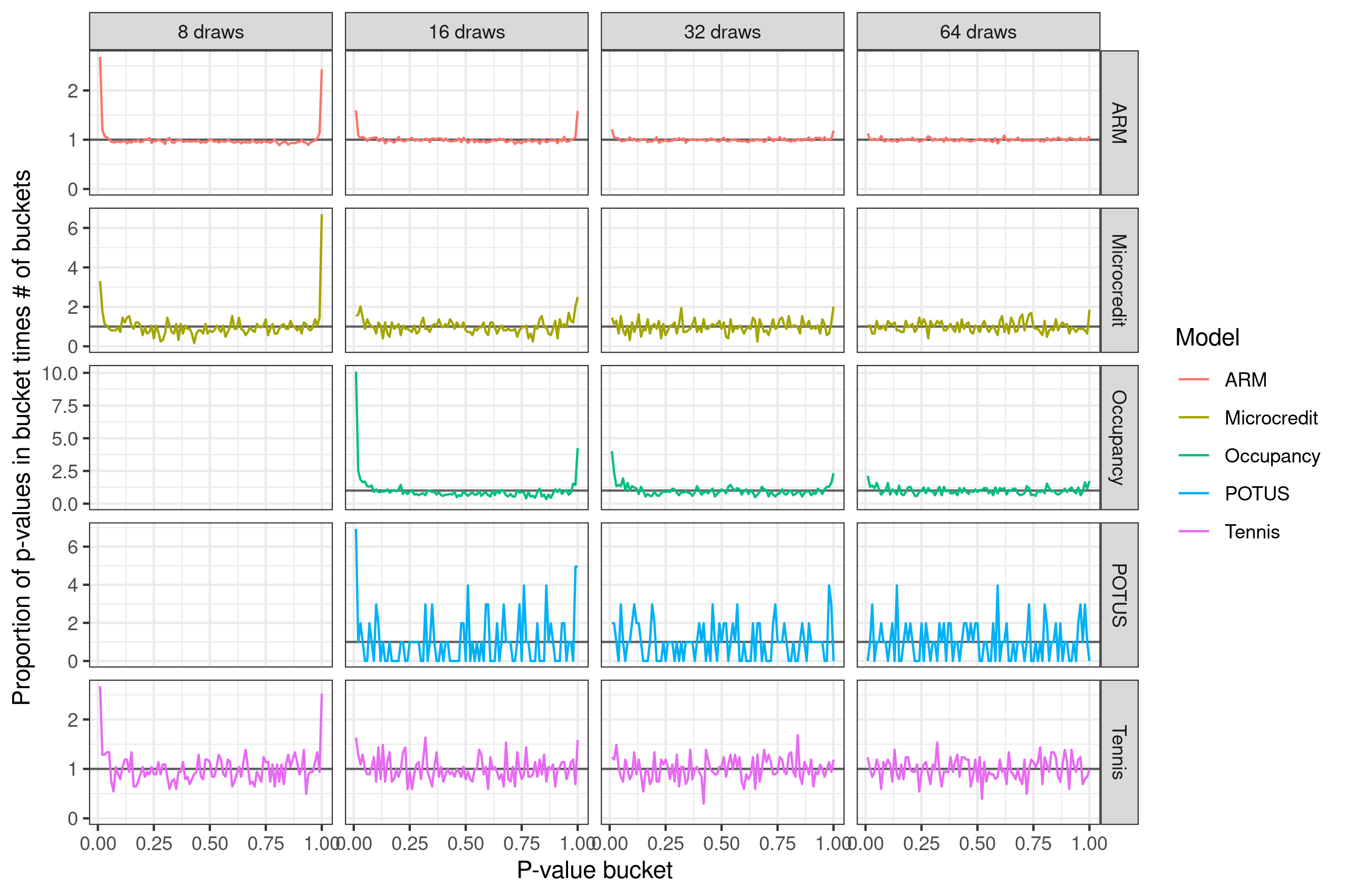} 

}

\caption[Density estimates of $\Phi(\freqerr)$ for difference models]{Density estimates of $\Phi(\freqerr)$ for difference models. All the ARM models are grouped together for ease of visualization.  Each panel shows a binned estimate of the density of $\Phi(\freqerr)$ for a particular model and number of draws $\znum$. Values close to one (a uniform density) indicate good frequentist performance.  CG failed for the Occupancy and POTUS models with only 8 draws, possibly indicating poor optimization performance with so few samples.}\label{fig:coverage}
\end{figure}

\end{knitrout}
}

\notbool{arxiv}{
    \ShortHeadings{BBVI with a Deterministic Objective}{Giordano, Ingram and Broderick}
    \firstpageno{1}
}

\begin{document}



\newcommand{\unnumberedfootnote}[2]{%
  \begingroup
  \renewcommand{\thefootnote}{}%
  \def\thefootnote{#1}\footnotetext{#2}\def\thefootnote{\arabic{footnote}}
  \endgroup
}

\title{Black Box Variational Inference with a Deterministic Objective: Faster, More Accurate, and Even More Black Box}

\notbool{arxiv}{
    \author{\name Ryan Giordano$^*$ \email rgiordano@berkeley.edu \\
    \addr Department of Statistics \\ 
    University of California \\
    Berkeley, CA, USA
    \AND
    \name Martin Ingram$^*$ \email martin.ingram@gmail.com  \\
    \addr Department of Biosciences\\
    University of Melbourne\\
    Parkville, VIC 3010, Australia
    \AND
    \name Tamara Broderick \email tbroderick@mit.edu \\
    \addr Department of Electrical Engineering and Computer
    Science \\ Massachusetts Institute of Technology \\
    Cambridge, MA, USA\\
  }
}{
    \author{Ryan Giordano,$^*$
       Martin Ingram,$^*$
       Tamara Broderick
     }
}

\maketitle

\unnumberedfootnote{*}{These authors contributed equally.}
\unnumberedfootnote{-}{The code to reproduce this paper can be found at 
\url{https://github.com/rgiordan/DADVIPaper}.}
\unnumberedfootnote{-}{Our Python implementation of DADVI can be found at \url{https://github.com/martiningram/dadvi}.}

\begin{abstract}
    Automatic differentiation variational inference (ADVI) offers fast and
easy-to-use posterior approximation in multiple modern probabilistic programming
languages. However, its stochastic optimizer lacks clear convergence criteria
and requires tuning parameters. Moreover, ADVI inherits the poor posterior
uncertainty estimates of mean-field variational Bayes (MFVB).  We introduce
``deterministic ADVI'' (DADVI) to address these issues. DADVI replaces the
intractable MFVB objective with a fixed Monte Carlo approximation, a technique
known in the stochastic optimization literature as the ``sample average
approximation'' (SAA).  By optimizing an approximate but deterministic
objective, DADVI can use off-the-shelf second-order optimization, and, unlike
standard mean-field ADVI, is amenable to more accurate posterior covariances via
linear response (LR).  In contrast to existing worst-case theory, we show that,
on certain classes of common statistical problems, DADVI and the SAA can perform
well with relatively few samples even in very high dimensions, though we also
show that such favorable results cannot extend to variational approximations
that are too expressive relative to mean-field ADVI. We show on a variety of
real-world problems that DADVI reliably finds good solutions with default
settings (unlike ADVI) and, together with LR covariances, is typically faster
and more accurate than standard ADVI.
\end{abstract}

\notbool{arxiv}{
    \begin{keywords}
        Black box variational inference,
        Mean field approximation,
        Automatic differentiation variational inference,
        Linear response covariances,
        Sample average approximation,
        Stochastic gradient
    \end{keywords}    
}

\section{Introduction}\label{sec:intro}
The promise of ``black-box'' Bayesian inference methods is that the user need
provide only a model and data. Then the black-box method should take care of
approximating the posterior distribution and reporting any summaries of interest
to the user. In settings where Markov chain Monte Carlo (MCMC) faces prohibitive
computational costs, users of Bayesian inference have increasingly turned to
variational methods. In turn, to improve ease of use in these cases, researchers
have developed a variety of ``black-box variational inference'' (BBVI) methods
\citep{ranganath:2014:bbvi, blei:2016:variational}. ``Automatic differentiation
variational inference'' (ADVI) represents a particularly widely used variant of
BBVI \citep{kucukelbir:2017:advi}, available in multiple modern probabilistic
programming languages.

However, researchers have observed that BBVI methods can face challenges with
both automation \citep{dhaka:2020:robuststochasticvi, welandawe:2022:robustbbvi}
and accuracy (\citealp[][Exercise 33.5]{mackay:2003:information};
\citealp[][Chapter 10.1.2]{bishop:2006:pattern}; \citealp{turner:2011:two};
\citealp[][Propositions 3.1--3.3]{huggins:2020:validated}). In particular, BBVI
takes an optimization-based approach to approximate Bayesian inference. The
optimization objective in a typical BBVI method involves an intractable
expectation over the approximating distribution. Most BBVI algorithms, including
ADVI, avoid computing the intractable expectation by using stochastic gradient
(SG) optimization, which requires only unbiased draws from the gradient of the
intractable objective. However, the use of SG is not without a price: SG
requires careful tuning of the step size schedule, can suffer from poor
conditioning, and convergence can be difficult to assess. On the accuracy side,
observe that ADVI minimizes the reverse Kullback-Leibler (KL) divergence over
Gaussian approximating distributions. The especially common mean-field variant
of this scheme, where the Gaussians are further constrained to fully factorize,
notoriously produces poor posterior covariance estimates (\citealp[][Exercise
33.5]{mackay:2003:information}; \citealp[][Chapter 10.1.2]{bishop:2006:pattern};
\citealp{turner:2011:two}), and research suggests variants beyond mean-field may
suffer as well \citep[][Proposition 3.2]{huggins:2020:validated}. In many cases,
these posterior covariance estimates can be efficiently corrected, without
fitting a more complex approximation, through a form of sensitivity analysis
known as ``linear response'' (LR) \citep{giordano:2015:lrvb,
giordano:2018:covariances}. However, LR cannot be used directly with SG, both
because the optimum is only a rough approximation and because the objective
function itself is intractable.

The stochastic optimization literature offers a well-studied alternative to SG:
the ``sample average approximation'' (SAA), which uses a single set of draws ---
shared across all iterations --- to approximate an intractable expected
objective. See \citet{kim:2015:guidetosaa} for a review of the SAA. In fact, a
number of papers have applied SAA to BBVI
\citep{giordano:2018:covariances,domke:2018:importanceweightingvi,domke:2019:divideandcouplevi,
broderick:2020:automatic,wycoff:2022:sparsebayesianlasso,giordano:2023:bnp}. But
before the present work and contemporaneous work by
\citet{burroni:2023:saabbvi}, there had not yet been a systematic study of the
efficacy of SAA for BBVI. \cite{burroni:2023:saabbvi} chooses an increasing
sequence of sample sizes in SAA, applied to variational inference with the
full-rank Gaussian approximation family, in order to achieve an increasingly
accurate approximation to the exact variational objective. In a complementary
vein, we here instead explore the promise and challenges of using SAA in BBVI
with a small, fixed number of samples --- with a focus on both automation and
accuracy. We call our method ``DADVI'' for ``deterministic ADVI,'' and we use
the unmodified ``ADVI'' to refer to the ADVI variational approximation optimized
with SG.

When considering a general optimization problem, the case for SAA over SG may at
first look weak. In full generality, SAA and SG require roughly the same number
of draws, $\znum$, for a particular accuracy. And the total number of draws required for a
given accuracy is expected to increase linearly in dimension \citep[][Chapter
5]{nemirovski:2009:sgdvsfixed,shapiro:2021:lectures}. Since SG uses each draw
only once, and SAA uses each draw at each step of a multi-step optimization
routine, SAA is, all else equal, expected to require more computation than SG in
the worst-case scenario, particularly in high dimensions
\citep{royset:2013:optimalsaa, kim:2015:guidetosaa}.  However, results in
particular cases can be quite different than these general conclusions.

We demonstrate that the SAA can be competitive with SG in BBVI applications both
theoretically and in experiments using real-world models and datasets.
Theoretically, we consider two cases common in Bayesian inference: (1) log
posteriors that are approximately quadratic, and (2) posteriors that have a
``global--local'' structure: roughly, there are some (global) parameters of
fixed dimension as the data set size grows, and some (local) parameters whose
dimension grows with the data cardinality. We further assume, as is typically
the case, that the user is interested in a relatively small number of quantities
of interest that are specified in advance, as opposed to, say, the maximum value
of a high-dimensional vector of posterior means.  In these cases, our theory
shows that DADVI does not suffer from the worst-case dimensional dependence that
the classical SAA literature suggests. In our experiments, we show that DADVI
produces competitive posterior approximations in very high-dimensional problems,
even with only $\znum = \DADVINumDraws$ draws, and even in models more complex
than the cases that we analyze theoretically.  Notably, in high dimensions, LR
covariances are considerably more computationally efficient --- and more
accurate --- than fitting a more complex variational approximation, such as a
full-rank normal. To our knowledge, the advantages of SAA for performing
sensitivity analysis, either within or beyond Bayesian inference, have not been
widely recognized.

Conversely, we show that SAA is not applicable to all BBVI methods. For example,
we show that, when using a full-rank ADVI approximation in high dimensions, the
SAA approximation leads to a degenerate variational objective unless the number
of draws used is very high --- on the order of the number of parameters.
The intuition behind how SAA fails in such a case applies to other highly
expressive BBVI approximations such as normalizing flows
\citep{rezende:2015:flows}.
In high dimensions, it is thus a combination of the relative paucity of the
mean-field ADVI approximation, together with special problem structure, that
makes DADVI a useful tool. Nevertheless, such cases are common enough that the
benefits of DADVI remain noteworthy.

In what follows, we start by reviewing ADVI (\cref{sec:setup}) and describing
DADVI (\cref{sec:dadvi}), our SAA approximation. We highlight how DADVI, unlike
ADVI, allows the use of LR covariances (\cref{sec:linear_response}). In
\cref{sec:mc_error_estimation}, we demonstrate how to approximately quantify
DADVI's Monte Carlo error, which arises from the single set of Monte Carlo
draws, and we note that such a quantification is not readily available for ADVI
due to its use of SG. We provide theory to support why DADVI can be expected to
work in certain classes of high-dimensional problems
(\cref{sec:high_dim_normal,sec:high_dim_global_local}), and we provide a
counterexample to demonstrate how DADVI can fail with very expressive BBVI
approximations (\cref{sec:dadvi_full_rank}). In a range of real-world examples
(\cref{sec:models_data}), we show that DADVI inherits the generally recognized
advantages of SAA, including the availability of off-the-shelf higher-order
optimization and reliable convergence assessment. We find experimentally that
DADVI, paired with LR covariances, can provide comparable posterior mean
estimates and more accurate posterior uncertainties
(\cref{sec:experiments_posterior_accuracy}) with less computation than
corresponding ADVI methods (\cref{sec:experiments_runtime}), including recent
work that endeavors to improve and automate the tuning of SG for BBVI
\citep{welandawe:2022:robustbbvi}, and we show that our estimates of Monte Carlo
sampling variability are accurate even for small values of $\znum$, around
$\DADVINumDraws$ (\cref{sec:experiments_sampling_variability}).



\section{Setup} \label{sec:setup}
In what follows, we take data $\y$ and a finite-dimensional parameter $\theta
\in \thetadom$. We consider a user who is able to provide software
implementations of the log density of the joint distribution $\p(\y, \theta)$
and is interested in reporting means and variances from an approximation of the
exact Bayesian posterior $\p(\theta \vert \y)$.

Black-box variational inference (BBVI) refers to a spectrum of approaches for
approximating this posterior. We focus in the present paper on ADVI, a
particularly popular instance of BBVI. Variational inference forms an
approximation $\q(\theta \vert \eta)$, with variational parameters $\eta \in
\etadom$, to $\p(\theta \vert \y)$. Let $\normal{\cdot}{\mu, \Sigma}$ denote a
normal distribution with mean $\mu$ and covariance matrix $\Sigma$. The
full-rank variant of ADVI approximately minimizes the reverse KL divergence
$\textrm{KL}\left( \q(\cdot \vert \eta) || \p(\cdot \vert \y)\right)$ between
the exact posterior and an approximating family of multivariate normal
distributions:
\begin{align}\label{eq:advi_qdom}
\qdom ={}& \left\{
    \q(\theta \vert \eta):
    \q(\theta \vert \eta) =
        \normal{\theta}{\mu(\eta), \Sigma(\eta)} \right\},
\end{align}
where $\eta \mapsto (\mu(\eta), \Sigma(\eta))$ is a (locally) invertible map
between the space of variational parameters and the mean and covariance of the
normal distribution. When we optimize $\eta$ over this family, we will refer to
the resulting optimization problem as the ``full-rank ADVI optimization
problem.''

In particular, we will typically focus on the following objective function,
which is equivalent to the one above:
\begin{align}\label{eq:kl_objective}
%
\klfullobj{\eta} :=
        \expect{\q(\theta \vert \eta)}{\log \q(\theta \vert \eta)}
        -\expect{\q(\theta \vert \eta)}{\logjoint}.
%
\end{align}
The objective $\klfullobj{\eta}$ in \cref{eq:kl_objective} is equivalent to the
KL divergence $\textrm{KL}\left( \q(\cdot \vert \eta) || \p(\cdot \vert
\y)\right)$ up to $\log \p(\y)$, which does not depend on $\eta$, so that
minimizing $\klfullobj{\eta}$ also minimizes the KL divergence. The negative of
the objective, $-\klfullobj{\eta}$, is sometimes called the ``evidence lower
bound'' (ELBO) \citep{blei:2016:variational}. 

To avoid degeneracy in the objective, typically one transforms any model
parameters with restricted range before running the optimization --- and
performs the reverse transformation after. E.g., we might take the logarithm of
any strictly positive parameters so that their transformed range is the full
real line; see \citet{kucukelbir:2017:advi} for further details. Therefore, we
will henceforth assume that $\thetadom = \rdom{\thetadim}$ and $\p(\theta)$ is
supported on all $\rdom{\thetadim}$. Then in the full-rank case, $\eta$ contains
both the mean and some unconstrained representation of a $\thetadim$-dimensional
covariance matrix, so that $\eta \in \rdom{\thetadim + \thetadim (\thetadim + 1)
/ 2}$.

The mean-field variant of ADVI restricts $\Sigma(\eta)$ to be
diagonal.\footnote{Technically, in the mean-field variant of ADVI,
$\Sigma(\eta)$ may sometimes be block diagonal; see \cref{app:mean-field}.
We elide this special case in what
follows for ease of exposition. Our experiments are fully diagonal.} That is,
take the approximating family $\qdom$ to consist of independent normals with
means $\etamu \in \rdom{\thetadim}$ and log standard deviations $\etaxi \in
\rdom{\thetadim}$.
\begin{align}\label{eq:advi_qdom_mf}
\qdom ={}& \left\{
    \q(\theta \vert \eta):
    \q(\theta \vert \eta) =
        \prod_{d=1}^{\thetadim}
            \normal{\theta_d}{\etamu[d], \exp(\etaxi[d])^2}
\right\} \\
\etamu ={}& (\etamu[1], \ldots, \etamu[\thetadim])^\trans
\textrm{, }
\etaxi = (\etaxi[1], \ldots, \etaxi[\thetadim])^\trans
\textrm{, }
\eta = (\etamu^\trans, \etaxi^\trans)^\trans
\textrm{, }
\etadom = \mathbb{R}^{\etadim}
\textrm{, and }
\etadim = 2\thetadim.
\nonumber
\end{align}
For mean-field ADVI, the variational parameter $\eta^\trans = (\etamu^\trans,
\etaxi^\trans) \in \rdom{2\thetadim}$. When the variational family satisfies the
mean-field assumption, we will refer to the resulting optimization problem as
the ``mean-field ADVI optimization problem'' and its objective as the
``mean-field ADVI objective.''

Using the expression for univariate normal entropy and neglecting some
constants, the mean-field ADVI objective in \cref{eq:kl_objective} becomes
\begin{align}\label{eq:kl_objective_advi}
\klfullobj{\eta} :=
-\sum_{d=1}^{\thetadim} \etaxi[d] -
\expect{\normal{\theta}{\eta}}{\log \p(\theta, \y)}
\quad\textrm{and}\quad
\etastar := \argmin_{\eta \in \etadom} \klfullobj{\eta}.
\end{align}
We would ideally like to compute $\etastar$, but we cannot optimize
$\klfullobj{\eta}$ directly, because the term
$\expect{\normal{\theta}{\eta}}{\log \p(\theta, \y)}$ is generally intractable.
ADVI, like most current BBVI methods, employs stochastic gradient optimization
(SG) to avoid computing $\klfullobj{\eta}$.  Specifically, ADVI uses Monte Carlo
and the ``reparameterization trick'' \citep{mohamed:2020:mcgradients} as
follows. Let $\normz$ denote the $\thetadim$-dimensional standard normal
distribution. If $\z \sim \normz$, then
%
\begin{align}\label{eq:reparameterization}
\expect{\normal{\theta}{\eta}}{\log \p(\theta, \y)} ={}&
    \expect{\normz}{\log \p(\etamu + \z \odot \exp(\etaxi), \y)}.
\end{align}
For compactness, we write 
\begin{align}\label{eq:reparameterization_theta}
    \theta(\eta, \z) := \etamu + \z \odot \exp(\etaxi),
\end{align}
where $\odot$ is the component-wise (Hadamard) product. 
For $\znum$ independent
draws\footnote{A subscript $\z_n$ will denote a particular member of the set
$\Z$, though for the rest of the paper, subscripts will usually denote an entry
of a vector.} $\Z := \{ \z_1, \ldots, \z_\znum \}$ from $\normz$, we can use
\cref{eq:reparameterization} to define an unbiased estimate for the mean-field
$\klfullobj{\eta}$:
\begin{align}\label{eq:reparameterization_kl}
\klobj{\eta | \Z} := 
-\sum_{d=1}^{\thetadim} \etaxi[d] -
\meann \log \p(\theta(\eta, \z_n), \y).
%
\end{align}
ADVI uses derivatives of $\klobj{\eta | \Z}$, with a new draw of $\Z$ at each
iteration, to estimate $\etastar$. The ADVI algorithm, which we will sometimes
refer to as ``ADVI'' in shorthand, can be found in \cref{alg:sadvi}. When we use
the ADVI algorithm for the full-rank optimization problem, we will write
``full-rank ADVI.''

\begin{minipage}{1.0\linewidth}
    \begin{minipage}[t]{0.49\linewidth}
        \begin{algorithm}[H]
        \caption{ADVI (Existing method)}\label{alg:sadvi}
        \begin{algorithmic}
        \Procedure{ADVI}{}
            \State $t \gets 0$
            \State Fix $\znum$ (typically $\znum = 1$)
            \While{Not converged}
                \State $t \gets t + 1$
                \State Draw $\Z$
                \State $\Delta \gets \grad{\eta}{\klobj{\eta_{t-1} | \Z}}$
                \State $\alpha_t \gets \textrm{SetStepSize(All past states)}$
                \State $\eta_t \gets \eta_{t-1} - \alpha_t \Delta$
                \State $\textrm{AssessConvergence(All past states)}$
            \EndWhile
            \State $\tilde{\eta} \gets \eta_t$ or
                $\tilde{\eta} \gets \frac{1}{M} \sum_{t'=t- M + 1}^t \eta_{t'}$
            \State \Return $\q(\theta \vert \tilde{\eta})$
        \EndProcedure

        \Postprocessing \hspace{1em}(If possible)
            \State Assess MC error using $\eta_1, \ldots, \eta_t$
            \If{MC Error is too high}
                \State Re-run with smaller / more steps
            \EndIf
        \EndPostprocessing
        \end{algorithmic}
        \end{algorithm}
    \end{minipage}
    \begin{minipage}[t]{0.49\linewidth}
        \begin{algorithm}[H]
        \caption{DADVI (Our proposal)}\label{alg:dadvi}
        \begin{algorithmic}
        \Procedure{DADVI}{}
            \State $t \gets 0$
            \State Fix $\znum$ (our default is $\znum = 30$)
            \State Draw $\Z$
            \While{Not converged}
                \State $t \gets t + 1$
                \State
    $\Delta \gets \textrm{GetStep}(\klobj{\cdot | \Z}, \eta_{t-1})$
                \State $\eta_t \gets \eta_{t-1} + \Delta$
                \State $\textrm{AssessConvergence}(\klobj{\cdot | \Z}, \eta_t)$
            \EndWhile
            \State $\etaopt \gets \eta_t$
            \State \Return $\q(\theta \vert \etaopt)$
        \EndProcedure
        \Postprocessing
            \State Compute LR covariances (\cref{sec:linear_response})
            \State Assess MC error (\cref{sec:mc_error_estimation})
            \If{MC Error is too high}
                \State Re-run with more samples in $\Z$
            \EndIf
        \EndPostprocessing
        \end{algorithmic}
        \end{algorithm}
    \end{minipage}
\end{minipage}

\vspace{1em}


\section{Our Method}\label{sec:dadvi}
Our method, DADVI, will start from the same optimization objective as ADVI, but
it will use a different approximation to handle the intractable objective. As we
have seen, in ADVI, each step of the optimization draws a new random variable.
The key difference in our method, DADVI, is that the random approximation is
instead made with a single set of draws and then fixed throughout optimization.
The full DADVI algorithm appears in \cref{alg:dadvi}. In the notation of
\cref{eq:reparameterization_kl}, for a particular $\Z$, 
the value $\etahat$ returned by DADVI in \cref{alg:dadvi} is given by
\begin{align}\label{eq:dadvi_estimate}
    \etahat := \argmin_{\eta \in \etadom} \klobj{\eta | \Z}.
\end{align}
The $\etahat$ of \cref{eq:dadvi_estimate} is an estimate of $\etastar$ insofar
as its objective $\klobj{\eta}$ is a random approximation to the true objective
$\klfullobj{\eta}$. In general, the idea of DADVI can be applied to either the
mean-field or full-rank ADVI optimization problem (though see
\cref{sec:dadvi_full_rank} below for some potential challenges when using DADVI
with full-rank ADVI).  In what follows, analogously to how we refer to the ADVI
algorithm, we will assume that we are targeting the mean-field problem with
DADVI unless explicitly stated that we are instead targeting the full-rank
problem.

For DADVI, the reparameterization of \cref{eq:reparameterization} is essential:
it allows us to use the same set of draws $\Z$ for any value of $\eta$. With
$\Z$ fixed, $\klobj{\cdot | \Z}$ in turn remains fixed throughout optimization.
This consistency would not be possible in general without a reparameterization
like \cref{eq:reparameterization} to separate the stochasticity from the shape
of $\q(\theta \vert \eta)$.  

Note that, for a given $\Z$, all derivatives of $\klobj{\eta | \Z}$ required by
either DADVI or ADVI can be computed using automatic differentiation and a
software implementation of $\log \p(\theta, \y)$.  In this sense, both DADVI and
ADVI are black-box methods. 
In practice, another key difference between ADVI and DADVI is that ADVI
typically draws only a single random variable per iteration, whereas DADVI uses
a larger number of draws; in particular, the default number of draws for DADVI
in our experiments will be $\znum = \DADVINumDraws$.

We will see in what follows that using DADVI instead of ADVI can reap large
practical benefits.

\subsection{Linear response covariances}\label{sec:linear_response}
We next review linear response (LR) covariances as an approximation for
posterior covariances of interest. We then show how DADVI accommodates LR
covariances in a way that ADVI does not. The key observation is that, since ADVI
does not actually minimize a tractable objective, sensitivity measures such as
LR covariances are not available, though they are for DADVI. To the authors'
knowledge, the availability of such sensitivity measures for SAA but not SG is
not yet a widely recognized advantage of SAA.

One well-documented failure of mean-field variational Bayes approximations
(including mean-field ADVI) is the mis-estimation of posterior variance
\citep{bishop:2006:pattern, turner:2011:two, giordano:2018:covariances,
margossian:2023:shrinkage}.  Even in cases for which mean-field approximations
provide good approximations to posterior means (e.g.\ when a Bayesian central
limit theorem can be approximately applied), the posterior variances are
typically incorrect. Formally, we often find that, for some quantity of interest
$\qoi{\theta} \in \mathbb{R}$,
\begin{align}\label{eq:mfvb_conceit}
\expect{\q(\theta \vert \etastar)}{\qoi{\theta}} \approx
\expect{\post}{\qoi{\theta}} \quad\textrm{but}\quad
\abs{\var{\q(\theta \vert \etastar)}{\qoi{\theta}} - 
\var{\post}{\qoi{\theta}}} \gg 0.
\end{align}
A classical motivating example is the case of multivariate normal
posteriors, which we review in \cref{sec:high_dim_normal}.

LR covariances comprise a technique for ameliorating the mis-estimation of
posterior variances without fitting a more expressive approximating class and
enduring the corresponding increase in computational complexity
\citep{giordano:2018:covariances}. Since posterior hyperparameter sensitivity
takes the form of posterior covariances, posterior covariances can be estimated
using the corresponding sensitivity of the VB approximation. Specifically, for
some $\phi_2(\theta)$, consider the exponentially tilted posterior, $\p(\theta
\vert \y, t) \propto \p(\theta \vert \y) \exp(t \phi_2(\theta))$.  When we can
exchange integration and differentiation, we find that
\begin{align}\label{eq:deriv_is_cov}
\p(\theta
\vert \y, t) \propto \p(\theta \vert \y) \exp(t \phi_2(\theta))
\quad\Rightarrow\quad
\fracat{d \expect{\p(\theta \vert \y, t)}{\phi_1(\theta)}}
       {dt}{t=0} = \cov{\p(\theta \vert \y)}{\phi_1(\theta), \phi_2(\theta)}.
\end{align}
A detailed proof of \cref{eq:deriv_is_cov} is given in Theorem 1 of
\cite{giordano:2018:covariances}; see also the classical score estimator of the
derivative of an expectation \citep{mohamed:2020:mcgradients}.  Together,
\cref{eq:mfvb_conceit,eq:deriv_is_cov} motivate the LR approximation
\begin{align}\label{eq:lr_def}
\lrcovfull{\q(\theta \vert \etastar)}{\phi_1(\theta), \phi_2(\theta)} :=
\fracat{d \expect{\q(\theta \vert \etastar(t))}{\phi_1(\theta)}}
       {dt}{t=0} =
\fracat{\partial \expect{\q(\theta \vert \eta)}{\phi_1(\theta)}}
       {\partial \eta^\trans}{\eta=\etastar}
\fracat{d \etastar(t)}{d t}{t = 0}
\end{align}
where $\etastar(t)$ minimizes the KL divergence to the tilted posterior
$\p(\theta \vert \y, t)$. By applying the implicit function theorem to the
first-order condition $\grad{\eta}{\klfullobj{\etastar}} = 0$, together with the
chain rule, \citet{giordano:2018:covariances} show that
\begin{align}\label{eq:lr_deriv}
\lrcovfull{\q(\theta \vert \etastar)}{\phi_1(\theta), \phi_2(\theta)}
=
\fracat{\partial \expect{\q(\theta \vert \eta)}{\phi_1(\theta)}}
      {\partial\eta^\trans}{\eta=\etastar}
\left(\hess{\eta}{\klfullobj{\etastar}} \right)^{-1}
\fracat{\partial \expect{\q(\theta \vert \eta)}{\phi_2(\theta)}}
     {\partial\eta}{\eta=\etastar}.
\end{align}

As discussed in \citet{giordano:2018:covariances} ---
and demonstrated in our experiments to follow --- it can often be the case that
$
\lrcovfull{\q(\theta \vert \etastar)}{\phi_1(\theta), \phi_2(\theta)}
\approx \cov{\post}{\phi_1(\theta), \phi_2(\theta)}
$,
even when $\cov{\q(\theta \vert \etastar)}{\phi_1(\theta), \phi_2(\theta)}$ is
quite a poor approximation to $\cov{\post}{\phi_1(\theta), \phi_2(\theta)}$. For
example, in the case of multivariate normal posteriors, the LR covariances are
exact, as we discuss in \cref{sec:high_dim_normal} below. See
\citet{giordano:2018:covariances} for more extended discussion of the intuition
behind \cref{eq:lr_def}.

Unfortunately, the derivative $d\etastar(t) / dt$ required by \cref{eq:lr_def}
cannot be directly computed for ADVI. 
%
First, observe that the Hessian matrix $\hess{\eta}{\klfullobj{\etastar}}$ in
\cref{eq:lr_deriv} cannot be computed for ADVI since neither $\etastar$ nor
$\klfullobj{\cdot}$ is computable.  One might instead approximate
$\hess{\eta}{\klfullobj{\eta}}$ with $\hess{\eta}{\klobj{\eta \vert \Z}}$ by
using additional Monte Carlo samples, and then evaluate at the ADVI optimum.
However, due to noise in the SG algorithm, the ADVI optimum typically does not
actually minimize $\klfullobj{\eta}$ nor $\klobj{\eta \vert \Z}$, so one is not
justified in applying the implicit function theorem at the ADVI optimum.

In contrast, DADVI does not suffer from these difficulties because its objective
function is available, and DADVI typically finds a parameter that minimizes that
objective to a high degree of numerical accuracy; one can ensure directly that
$\etahat$ is, to high precision, a local minimum of $\klobj{\eta | \Z}$. 
Therefore, we are justified  in applying the implicit
function theorem to the first-order condition $\grad{\eta}{\klobj{\eta \vert
\Z}} = 0$.  If we follow the derivation of \cref{eq:lr_deriv} but
with $\klobj{\eta \vert \Z}$ in place of $\klfullobj{\eta}$, we find
the following tractable LR covariance estimate:
\begin{align}\label{eq:lr_deriv_dadvi}
\lrcov{\q(\theta \vert \etahat)}{\phi_1(\theta), \phi_2(\theta)}
={}&
\fracat{\partial \expect{\q(\theta \vert \eta)}{\phi_1(\theta)}}
      {\partial\eta^\trans}{\eta=\etahat}
\h^{-1}
\fracat{\partial \expect{\q(\theta \vert \eta)}{\phi_2(\theta)}}
     {\partial\eta}{\eta=\etahat}  \\
\quad\textrm{where}\quad
     \h :={}& \hess{\eta}{\klobj{\eta \vert \Z}}. \nonumber
\end{align}

%

We note that the same reasoning that leads to a tractable version of LR
covariances applies to other sensitivity measures, such as prior sensitivity
measures \citep{giordano:2023:bnp} or the infinitesimal jackknife
\citep{giordano:2019:swiss}. Though we do not explore these uses of sensitivity analysis
in the present work, one expects DADVI but not ADVI to support such analyses.

\subsection{Monte Carlo error
estimation}\label{sec:mc_error_estimation}
In this section we show how to estimate the Monte Carlo error of the output of
DADVI. Since this estimate is based on use of the implicit function theorem, as
we saw for LR covariances, it is again not as readily available to ADVI.

Let $\fun(\eta)$ denote some quantity of interest, such as a posterior
expectation of the form $\fun(\eta) = \expect{\q(\theta \vert
\eta)}{\phi_1(\theta)}$ as in the previous \cref{sec:linear_response}.  We are
now interested in the sampling variance of $\fun(\etahat) - \fun(\etastar)$ due
to the Monte Carlo randomness in $\Z$. We can apply standard asymptotic theory
for the variance of M-estimators to find that this sampling variance is, in the
notation of \cref{sec:linear_response}, consistently estimated by
\begin{align}\label{eq:mc_variance}
\var{\normz}{\fun(\etahat) - \fun(\etastar)}
\approx{}&
\frac{1}{\sqrt{N}} \grad{\eta}{\fun(\etaopt)}^\trans
\h^{-1} \scorecov \h^{-1}
\grad{\eta}{\fun(\etaopt)},  \\
\quad\textrm{where}\quad
\scorecov :={}&
\meann \grad{\eta}{\klobj{\etahat | \z_n}}
\grad{\eta}{\klobj{\etahat | \z_n}}^\trans. \nonumber
\end{align}
\Cref{eq:mc_variance} is analogous to the ``sandwich covariance'' estimate for
misspecified maximum likelihood models \citep{stefanski:2002:mestimation}.
Indeed, the question of how variable the DADVI estimate $\etahat$ is under
sampling of $\Z$ is exactly the same as asking how variable a misspecified
maximum likelihood estimator (or any M-estimator) is under sampling of the data,
and the same conceptual tools can be applied.  To complete the analogy, our
$\scorecov$ plays the role of the empirical score covariance, and $\h$ plays the
role of the empirical Fisher information.

Analogously to our discussion of LR covariances in \cref{sec:linear_response},
we briefly note that the classical derivation of \cref{eq:mc_variance} is based
on a Taylor series expansion of the first-order condition
$\grad{\eta}{\klobj{\eta \vert \Z}} = 0$, and so is not applicable to
estimators like ADVI that do not satisfy any computable first-order conditions.

\subsection{Computational considerations}
\label{sec:lr_mc_computation} 
We next describe best practices in computing LR covariances and the Monte Carlo
sampling variability of the DADVI estimate.  First, we delineate how to use
these quantities to check that the number of samples $\znum$ is adequate.
Second, we discuss how to handle the primary computational difficulty of
computing both quantities, namely the inverse of the Hessian matrix of the DADVI
objective at the optimum.

In the postprocessing step of \cref{alg:dadvi}, we recommend computing both
\cref{eq:lr_deriv_dadvi} and \cref{eq:mc_variance} for each quantity of
interest. One might consider a Bayesian analysis non-robust to sampling
uncertainty if decisions based on the Bayesian analysis might change due to the
sampling uncertainty. For instance, in a typical Bayesian analysis, one might
make decisions based on how far a posterior mean is from a decision boundary in
units of posterior standard deviation. Therefore, we might expect that sampling
variability could be decision-changing if the estimated sampling variability
dominated the estimated posterior uncertainty. In turn, then, we recommend using
a comparison of the estimated quantities from
\cref{sec:linear_response,sec:mc_error_estimation} to check the adequacy of the
sample size $\znum$. If the estimated sampling variability dominates or might
generally be sufficiently large as to be decision-changing, we recommend
increasing $\znum$. In the present work we will not attempt to formalize nor to
analyze such a procedure, although \citet{burroni:2023:saabbvi} and the general
SAA literature \citep{royset:2013:optimalsaa, kim:2015:guidetosaa} attempt to
estimate properties of the optimization and optimally allocate computing
resources in a schedule of increasing sample sizes.

In \cref{eq:lr_deriv_dadvi,eq:mc_variance}, the quantities $\scorecov$ and
$\grad{\eta}{\fun(\etaopt)}$ are typically straightforward to efficiently
compute with automatic differentiation, but direct computation of $\h^{-1}$
would incur a computational cost on the order of roughly $\etadim^3$, which can
be prohibitive in high-dimensional problems.  However, for a given quantity of
interest $f(\eta)$, it suffices for both \cref{eq:lr_deriv_dadvi,eq:mc_variance}
to compute the $\etadim$-vector $\h^{-1} \grad{\eta}{\fun(\etaopt)}$. For models
with very large $\etadim$, we recommend evaluating $\h^{-1}
\grad{\eta}{\fun(\etaopt)}$ using the conjugate gradient method, which requires
only Hessian-vector products of the form $\h v$ \citep[Chapter
5]{nocedal:1999:optimization}. These products can be evaluated quickly using
standard automatic differentiation software.  As long as the number of
quantities of interest is not large, both LR and sampling uncertainties can thus
be computed at considerably less computational cost than a full matrix
inversion.


\section{Considerations in high dimensions}\label{sec:high_dim}
As discussed in \cref{sec:intro}, classical analysis in the optimization
literature argues that, in the worst case, SAA is expected to require more
computation than SG for a given approximation accuracy.  The reason is that the
total number of samples required for a given accuracy scales linearly with
dimension, for both SG and SAA \citep[][Chapter
5]{nemirovski:2009:sgdvsfixed,shapiro:2021:lectures}. Since SAA requires more
computation per sample than SG, one would correspondingly expect SAA to require
more computation than SG for the same accuracy in high dimensional problems.

In this section we discuss why the aforementioned classical analysis of
dimension dependence does not necessarily apply to the particular structure of
the mean-field ADVI problem and some of its typical applications.  We argue
that, for problems that are approximately normal, or problems that are high
dimensional due only to a having a large number of low-dimensional ``local''
parameters, DADVI can be effective with a relatively small number of samples
which, in particular, need not grow linearly as the dimension of the problem
grows.  In contrast, we show that SAA may be inappropriate for more expressive
BBVI approximations, such as full-rank ADVI. A key assumption of our analysis is
that the user is interested in a relatively small number of scalar-valued
quantities of interest, even though these quantities of interest may depend in
some sense on the whole variational distribution.

\subsection{High dimensional normals}\label{sec:high_dim_normal}
We will show that in the normal model, the number of samples required to
estimate any particular posterior mean do not depend on the dimension.  Further,
the LR covariances from DADVI are exact, irrespective of the problem dimension,
and are in fact independent of the particular $\Z$ used.  Conversely, the
worst-case error in the DADVI posterior mean estimates across all dimensions
will grow as dimension grows.

Take the quadratic model
\begin{align}\label{eq:quadratic_model}
\logjoint =
    -\frac{1}{2} \theta^\trans A \theta + B^\trans \theta =
    -\frac{1}{2} \trace{A \theta \theta^\trans } + B^\trans \theta
\end{align}
for a known matrix $A \in \rdom{\thetadim \times \thetadim}$ and vector $B \in
\rdom{\thetadim}$, possibly depending on the data $\y$.  Such a model arises,
for example, when approximating the posterior of a conjugate normal location
model, in which case the posterior mean $A^{-1} B$ and covariance matrix
$A^{-1}$ would depend on the sufficient statistics of the data $\y$.
Additionally, as we show below, the exact variational objective is available in
closed form for the quadratic model. Of course, there is no need for a
variational approximation to a posterior which is available in closed form, nor
any need for a stochastic approximation to a variational objective which is
available in closed form. However, studying quadratic models can provide
intuition for the dimension dependence of DADVI approximations when the problem
is {\em approximately} quadratic.

We first derive the exact variational objective function and its optimum. Recall
our notation of \cref{sec:setup}, in which the variational posterior mean is
denoted $\etamu$ and the log standard deviation is denoted $\etaxi$.  For
compactness, we additionally write the vector of variational standard deviation
parameters as $\etasigma = \exp(\etaxi)$, where $\exp(\cdot)$ is applied
component-wise. Let $\etasigma^2$ be the corresponding vector of variance
parameters. Note that
\begin{align*}
\expect{\q(\theta \vert \eta)}{\theta} ={}& \etamu
\quad\textrm{and}\quad
\expect{\q(\theta \vert \eta)}{\theta \theta^\trans} ={}
    \etamu \etamu^\trans + \diag{\etasigma^2}, \\
\textrm{so } \quad
\klfullobj{\eta} ={}&
    \frac{1}{2} \etamu^\trans A \etamu
    -\frac{1}{2} \etasigma^\trans (A \odot I_{\thetadim}) \etasigma
    - B^\trans \mu
    -\sumd \log \etasigma[d].
\end{align*}
The exact optimal parameters are thus
\begin{align*}
\etamustar = A^{-1}B
\quad\textrm{and}\quad
\etasigmastar[d] = (A_{dd})^{-1/2}.
\end{align*}
\def\zbar{\bar{z}}
\def\zzbar{\overline{zz^\trans}}
\def\sigmat{S}
\def\zcov{\hat{\Sigma}_{z}}
If the objective had arisen from a multivariate normal posterior, observe the
variational approximation to the mean is exactly correct, but the covariances
are, in general, mis-estimated, since $1 / A_{dd} \ne (A^{-1})_{dd}$ unless the
true posterior covariance is diagonal.

We next make an asymptotic argument that we can expect any particular DADVI output
to be a good estimate of the optimum of the exact objective, even for a small $\znum$.
\begin{proposition}\label{prop:normal_accurate}
Consider any parameter dimension index $d \in \{1, \ldots, \thetadim\}$,
selected independently of $\Z$. In the quadratic model, we have
$\etasigmahat[d]^{-2} - \etasigmastar[d]^{-2}= O_p(\znum^{-1/2})$ and
$\etamuhat[d] - \etamustar[d] = O_p(\znum^{-1/2})$. The constants do not depend
on $\thetadim$.
\end{proposition}

\begin{proof}
We can compare the optimal parameters with the DADVI estimates.  Let $\zbar :=
\meann \z_n$.  Let $\dequal$ denote equality in distribution and let $Q \sim
\chi^2_{\znum-1}$ denote a chi-squared random variable with $\znum-1$ degrees of
freedom. We show in \cref{app:normal_accurate_proof} that, irrespective of the
dimension of the problem,
\begin{align*}
\etamuhat = \etamustar - \etasigmahat \odot \zbar
\quad\textrm{and}\quad
\etasigmahat[d] \dequal \left(\frac{Q}{\znum} A_{dd} \right)^{-1/2}.
\end{align*}
So $\expect{\normz}{\etasigmahat[d]^{-2}} = \frac{\znum - 1}{\znum} A_{dd} =
\frac{\znum - 1}{\znum} \etasigmastar[d]^{-2}$, and $\etasigmahat[d]^{-2} -
\etasigmastar[d]^{-2}= O_p(\znum^{-1/2})$.  It follows that $\etamuhat[d] -
\etamustar[d] = O_p(\znum^{-1/2})$ as well.
\end{proof}

The next remark suggests that, in cases with $N \ll \thetadim$, the
worst-estimated linear combination of means is poorly estimated; note that in
choosing the worst case we can overfit the draws $\Z$. It follows that the
behaviors for any particular element of $\etamuhat$ and $\etasigmahat$ above do
not imply that DADVI performs uniformly well across all parameters. 
\begin{remark} \label{rem:worst_case}
In the quadratic model, we have
\begin{align*}
\expect{\normz}{
    \sup_{\nu: \norm{\nu}_2 = 1}
    \nu^\trans \frac{\etamuhat - \etamustar}{\etasigmahat}} =
\expect{\normz}{
    \sup_{\nu: \norm{\nu}_2 = 1} \nu^\trans \zbar} =
    \expect{\normz}{\sqrt{\zbar^\trans \zbar}} \approx
    \sqrt{\frac{\thetadim}{N}}.
\end{align*}
In the first term of the preceding display, the division in the term $(\etamuhat
- \etamustar) / \etasigmahat$ is elementwise. The final relation follows since
$\sqrt{N} \zbar$ is a $\thetadim$-dimensional standard normal, so $N
\zbar^\trans \zbar$ is a $\chi^2_{\thetadim}$ random variable.
\end{remark}

Finally, we show that the LR covariances reported from DADVI
are exact, regardless of how small $\znum$ is or, indeed, the particular values of $\Z$.
Recall that, by contrast, the exact mean-field variance estimates are
notoriously unreliable as estimates of the posterior variance. 
\begin{proposition}\label{prop:mvn_lr}
In the quadratic model, we have
\begin{align*}
\lrcov{\q(\theta \vert \etahat)}{\theta} =
\fracat{d \etamuhat}{d t^\trans}{\etahat} = A^{-1},
\end{align*}
with no $\Z$ dependence.
\end{proposition}
See \cref{app:mvn_lr_proof} for a proof.
Since $A^{-1}$ is in fact the exact posterior variance,
the linear response covariance is exact in this case irrespective of how small
$\znum$ is, in contrast to $\etasigmastar$, which can be a poor estimate of
the marginal variances unless $A$ is diagonal.

\subsection{High dimensional local variables}\label{sec:high_dim_global_local}
\def\err{\mathcal{E}}
\newcommand{\f}{f} 
\newcommand{\fbar}{\bar{f}} 
\newcommand{\fhat}{\hat{f}} 
\newcommand{\ftil}{\tilde{f}} 

We next show that the number of samples required for DADVI estimation grows only
logarithmically in dimension when the target joint distribution can be written
as a large number of nearly independent problems that share a single,
low-dimensional global parameter. 

Formally, we say a problem has a ``global--local'' structure if we have the
following decomposition:\footnote{Each local parameter is, itself, a vector, so
we use superscripts to distinguish local parameters, retaining subscripts for
particular elements of vectors.}
\begin{align}\label{eq:global_local}
\theta = \begin{pmatrix}
\gamma\\
\lambda^1 \\
\vdots\\
\lambda^P
\end{pmatrix}
\quad\textrm{and}\quad
\logjoint = \sump \ell^p(\gamma, \lambda^p) + \ell^\gamma(\gamma),
\end{align}
where $\lambda^p \in \rdom{\lambdadim}$ and $\gamma \in \rdom{\gammadim}$, and
any data dependence is implicit in the functions $\ell^p$ and $\ell^\gamma$.
Here, the ``global'' parameters $\gamma$ are shared among all ``observations,''
and the ``local'' $\lambda^p$ parameters do not occur with one another.  We
assume that the dimensions $\gammadim$ and $\lambdadim$ are small, but that the
total dimension $\thetadim = \gammadim + P \lambdadim$ is large because $P$ is
large, i.e., because there are many local parameters.

Each vector involved in the variational approximation --- the variational
parameter $\eta$, the variational mean $\etamu$ and standard deviation
$\etasigma$, the normal random variables $\z$, and the sets $\Z$ of normal
random variables --- can be partitioned into sub-vectors related to the global
and local parameters. We will denote these subvectors with $\gamma$ and $p$
superscripts, respectively, so that, e.g., $\eta^\trans = (\eta^\gamma, \eta^1,
\ldots, \eta^p, \ldots, \eta^P)$, and so on.  We will write $\etadom^\gamma$ for
the domain of $\eta^\gamma$ and $\etadom^p$ for the domain of $\eta^p$.

If there were no global parameters $\gamma$, then the high dimensionality would
be no problem for DADVI.  Without shared global parameters, the variational
objective would consist of $P$ completely independent $\lambdadim$-dimensional
sub-problems.  According to the classical optimization results referred to at
the beginning of this section (e.g. \citet[Chapter 5]{shapiro:2021:lectures}),
under typical regularity conditions, each of these sub-problems' solutions could
be accurately approximated with DADVI using no more than $\znum=O(\lambdadim)$
standard normal draws, each of length $\lambdadim$.  The corresponding $\Z$ for
the combined problem would stack the $\znum$ vectors for each sub-problem,
resulting in a $\Z$ consisting again of only $\znum$ standard normal draws, each
of length $P \lambdadim$.  For this combined problem, any particular posterior
mean of the combined problem (chosen independently of $\Z$) would then be
well-estimated using only $\znum$ draws, although we would expect more
adversarial quantities such as $\max_{p} \sup_{v:\norm{v}_2=1}
v^\trans\left(\hat{\eta}^p - \etastar^{p}\right)$ to be poorly estimated, as we
saw in the quadratic problem (see \cref{rem:worst_case} in
\cref{sec:high_dim_normal} above).

The goal of the present section is to state conditions under which the extra
dependence induced by the shared finite-dimensional global parameter does not
depart too strongly from the fully independent case described in the preceding
paragraph.  Our two key assumptions, stated respectively in
\cref{assu:local_ulln,assu:local_minimum} below, are that each local problem
obeys a sufficiently strong uniform law of large numbers, and that the local
problems do not, in a certain sense, provide contradictory information about the
global parameters.

To state our assumptions, let us first introduce some notation.  Similar to around
\cref{eq:reparameterization}, we write $\gamma(\eta^\gamma, \z^\gamma) =
\etamu^\gamma + \exp(\etaxi^\gamma) \odot \z^\gamma$, with analogous notation
for $\lambda^p(\eta^p, \z^p)$.  

Our first step is to write the variational objectives as the sum of
``local objectives.''
\begin{defn}\label{defn:local_objective}
Define the ``local objective''
\begin{align*}
\f^p(\eta^\gamma, \z^\gamma, \eta^p, \z^p) :={}&
    \ell^p\left(\gamma(\eta^\gamma, \z^\gamma),
                \lambda^p(\eta^p, \z^p)\right) +
    \sum_{d=1}^{\lambdadim} {\etaxi}^p_d +
    \frac{1}{P} \left(
        \ell^\gamma(\gamma(\eta^\gamma, \z^\gamma)) +
            \sum_{d=1}^{\gammadim} {\etaxi}^\gamma_d
    \right).
\end{align*}
We then define its expected value $\fbar^p$ and corresponding sample
approximation $\fhat^p$:
\begin{align*}
\fbar^p(\eta^\gamma, \eta^p) :={}&
    \expect{\normz}{\f^p(\eta^\gamma, \z^\gamma, \eta^p, \z^p)}, 
&\fhat^p(\eta^\gamma, \Z^\gamma, \eta^p, \Z^p) :={}&
\meann \f^p(\eta^\gamma, \z^\gamma_n, \eta^p, \z^p_n).
\end{align*}
\end{defn}
With these definitions in hand, we observe that the mean-field objective
for this model and its sample approximation can be written as functions
of the local quantities.
\begin{align*}
\klfullobj{\eta} ={}&
    -\sump \fbar^p(\eta^\gamma, \eta^p),
&
\klobj{\eta | \Z} ={}&
    -\sump \fhat^p(\eta^\gamma, \Z^\gamma, \eta^p, \Z^p).
\end{align*}

Our key assumption is that a sub-Gaussian uniform law of numbers (ULLN) applies
to each local objective.
%
\begin{assu}[A uniform law of large numbers applies to the local problems]
\label{assu:local_ulln}
Assume that, for any $\delta > 0$, there exist positive constants $C_1$, $C_2$,
and $N_0$ depending on $\lambdadim$ and $\gammadim$ but not on $P$ such that for
$\znum \ge N_0$,
\begin{align*}
\p\left(
\sup_{\left( \eta^\gamma, \eta^p \right) \in \etadom^{\gamma} \times \etadom^p}
\abs{\fhat^p(\eta^\gamma, \Z^\gamma, \eta^p, \Z^p) -
     \fbar^p(\eta^\gamma, \eta^p)} > \delta
\right)
\le
\varepsilon := C_1 \exp\left(-C_2 N_0 \right).
\end{align*}
\end{assu}
%
%
%
\begin{example}\label{ex:ulln_compact}
Recall the definition of the ``local objective'' given in
\cref{defn:local_objective}.  
Assume that $\etadom^{\gamma} \times \etadom^p$ is compact and $\fbar^p$ is
Lipschitz. Assume that, for all parameters in $\etadom^{\gamma} \times
\etadom^p$, the moment generating function of $\fhat^p(\eta^\gamma, \z^\gamma,
\eta^p, \z^p)$ is finite in a neighborhood of $0$, and that
$\var{\z}{\fhat^p(\eta^\gamma, \z^\gamma, \eta^p, \z^p)}$ is finite. Then
\citet[Theorem 12 and Equation 3.17]{shapiro:2003:montecarlosampling} implies that
\cref{assu:local_ulln} holds.\footnote{The connection between our notation and
Shapiro's is as follows. Shapiro's $\alpha$ is our $1 - \varepsilon$. Shapiro's
$\varepsilon$ is our $\delta$. Shapiro's $\delta = 0$ in our case because we
assume that $\etahat$ is an exact optimum. Shapiro's diameter $D$ is bounded
because $\etadom^{\gamma} \times \etadom^p$ is compact. Shapiro's $L$ is our
Lipschitz constant. Shapiro's $n$ is our $\gammadim + \lambdadim$. And Shapiro's
$\sigma^2_{\textrm{max}}$ is bounded by our assumption on the variance of
$\fhat^p$.  A similar but more detailed result can also be found in
\citet[Section 5.3.2]{shapiro:2021:lectures}.}
\end{example}
Though restrictive, the conditions of \cref{ex:ulln_compact} are those that give
rise to the commonly cited linear dimensional dependence for the SAA
\citep[e.g.][]{nemirovski:2009:sgdvsfixed, kim:2015:guidetosaa,
homem:2014:montecarlosaa}. Similar conditions to \cref{ex:ulln_compact} can be
also found in the statistics literature.  For example, \citet[Theorem
4.10]{wainwright:2019:high} provides a bound of the form in
\cref{assu:local_ulln} for bounded $f^p$ with Rademacher complexity that
decreases in $\znum$.
Note that ADVI objectives, like many maximum likelihood problems, are typically
over unbounded domains, with non-Lipschitz objective functions.  In such cases,
one can still use \cref{assu:local_ulln} by showing first that an estimator
converges suitably quickly to a compact set with high probability, and then use
\cref{assu:local_ulln} on that compact set; see, e.g., the discussion in Section
3.2.1 of \citet{wellner:2013:empiricalprocesses}.
Our present purpose is not to survey the extensive literature on circumstances
under which \cref{assu:local_ulln} holds, only to demonstrate simple,
practically relevant conditions under which the SAA does not suffer from the
worst-case dimensional dependence suggested by the SAA literature.



Next, we assume that the optima are well-defined for the local
problems.
%
%
\begin{assu}[A strict minimum exists] \label{assu:local_minimum}
Assume that there exists a strict optimum at $\etastar$ in the sense that there
exists a positive constant $C_3$, not depending on $P$, that satisfies
\begin{align*}
\klfullobj{\eta} - \klfullobj{\etastar} \ge
    P C_3 \norm{\eta^\gamma - \etastar^\gamma}_2^2
\quad\textrm{ and }\quad
\klfullobj{\eta} - \klfullobj{\etastar} \ge
    C_3 \sum_{p=1}^P \norm{\eta^p - \etastar^p}_2^2.
\end{align*}
\end{assu}
%

As illustrated by \cref{ex:local_minimum} below, a key aspect of
\cref{assu:local_minimum} is that each local objective function is informative
about the global parameter, so that as the dimension $P$ grows, the global
objective function grows ``steeper'' as a function of $\eta^\gamma$.

\begin{example}\label{ex:local_minimum}
Recall \cref{defn:local_objective}.  
Suppose that, for each $p$, the expected local objective $\fbar^p(\eta^\gamma,
\eta^p)$ is twice-differentiable and uniformly convex, in the sense that there
exists a lower bound $C_3 > 0$ on the eigenvalues of the second derivative matrices
of $\fbar^p(\eta^\gamma, \eta^p)$, uniformly in both $p$ and $\eta$.  Then, by a
Taylor series expansion,
\begin{align*}
\klfullobj{\eta} - \klfullobj{\etastar} \ge
    C_3 \left( P  \norm{\eta^\gamma - \etastar^\gamma}_2^2 +
    \sump \norm{\eta^p - \etastar^p}_2^2
    \right),
\end{align*}
from which \cref{assu:local_minimum} follows.  (See
\cref{app:high_dim_global_local} for more details.)
\end{example}
%

%
\begin{theorem}\label{thm:global_local}
Under \cref{assu:local_ulln,assu:local_minimum}, for any $\varepsilon > 0$ and
$\delta > 0$, there exists an $N_0$, depending only logarithmically on $P$,
such that $\znum \ge N_0$ implies that
\begin{align*}
\p\left(
\norm{\etahat^\gamma - \etastar^\gamma}_2^2 \le \delta
\textrm{ and, for all }p\textrm{, }
        \norm{\etahat^p - \etastar^p}_2^2 \le \delta
\right)
\ge 1- \varepsilon.
\end{align*}
\begin{proof}[sketch]
By \cref{assu:local_minimum}, closeness of $\fbar^p(\eta^\gamma, \eta^p)$ and
$\fhat^p(\eta^\gamma, \Z^\gamma, \eta^p, \Z^p)$ implies closeness of $\etahat^p$
and $\etastar^p$, and closeness of $\frac{1}{P} \klobj{\eta | \Z}$ and
$\frac{1}{P} \klfullobj{\eta}$ implies closeness of $\etahat^\gamma$ and
$\etastar^\gamma$.  Thus, for $\etahat$ to be close to $\etastar$, it suffices
for $\abs{\fbar^p(\eta^\gamma, \eta^p) - \fhat^p(\eta^\gamma, \Z^\gamma, \eta^p,
\Z^p)} < \delta'$ simultaneously for all $p$, and for some $\delta'$ that is a function of
$\delta$ and the constants in \cref{assu:local_minimum}. To apply a union bound
to \cref{assu:local_ulln} requires decreasing $\varepsilon$ by a factor of $P$,
which requires increasing $\znum$ by a factor of no more than $\log P$.

See \cref{app:high_dim_global_local} for a detailed proof.
\end{proof}
\end{theorem}

The key difference between classical results such as \citet[Chapter
5]{shapiro:2021:lectures} and our \cref{thm:global_local} is that, in the
classical results, $\znum = O(P)$, whereas for \cref{thm:global_local}, $\znum =
O(\log P)$.  Intuitively, $\znum$ need grow only logarithmically in $P$ because
the global parameters are sharply identified, which approximately decouples the
remaining local problems.

\subsection{DADVI fails for full-rank ADVI}
\label{sec:dadvi_full_rank}
The preceding sections demonstrated that, in certain cases, DADVI can work well
to estimate the optimum of the mean-field ADVI problem even in high dimensions.
By contrast, we now show that DADVI will behave pathologically for the full-rank
ADVI problem in high dimensions unless a prohibitively large number of draws are
used. The intuition we develop for full-rank ADVI also extends to other highly
expressive variational approximations such as normalizing flows.

In forming the full-rank optimization problem, ADVI parameterizes $\q(\theta |
\eta)$ using a mean $\etamu$ and a $\thetadim \times \thetadim$ matrix $\etar$
in place of $\etasigma$.  Formally, the full-rank approximation taking
$\theta(\eta, \z) = \etamu + \etar \z$ in place of the mean-field
reparameterization is given in \cref{eq:reparameterization_theta}. Letting
$|\cdot|$ denote the matrix determinant, the KL divergence becomes
\begin{align}\label{eq:reparameterization_fr_kl}
    \klobj{\eta | \Z} := -\frac{1}{2} \log |\etar \etar^\trans| -
        \meann \log \p(\etamu + \etar \z_n, \y).
\end{align}
The preceding display can be compared with the corresponding mean-field
objective in \cref{eq:reparameterization_kl}.
For the present section, we will
write $\klobj{\eta | \Z} = \klobj{(\mu, \etar) | \Z}$.  Under this
parameterization, $\cov{\q(\theta \vert \eta )}{\theta} =
\cov{\normz}{\theta(\eta, \z)} = \etar \etar^\trans$, so the matrix $\etar$ can
be taken to be any square root of the covariance matrix of $\q(\theta \vert
\eta)$.  In practice, $\etar$ is typically taken to be lower-triangular (i.e., a
Cholesky decomposition), though the particular form of the square root used will
not matter for the present discussion.

Suppose we are attempting to optimize the full-rank ADVI problem with DADVI when
$\thetadim > \znum$, so that $\theta$ has more dimensions than there are draws
$\z_n$.  Our next result shows that, in such a case, DADVI will behave
pathologically.

\begin{theorem}\label{thm:fradvi}
Consider a full-rank ADVI optimization problem with $\thetadim > \znum$. Then,
for any $\mu$, we have $\inf_{\etar} \klobj{(\mu, \etar) | \Z} = -\infty$, so
the DADVI estimate is undefined.
\begin{proof}
In the full-rank case, the objective function $\klobj{\eta | \Z}$ in
\cref{eq:reparameterization_fr_kl} depends on $\etar$ only through the products
$\etar \z_n$ and the entropy term, which is $\frac{1}{2} \log |\etar
\etar^\trans| = \log |\etar|$.  Since $\znum < \thetadim$, we can write $\etar =
\etar^{\Z} + \etar^{\perp}$, where $\etar^{\Z}$ is a rank-$\znum$ matrix
operating on the subspace spanned by $\Z$ and $\etar^\perp$ is a
rank-$(\thetadim - \znum)$ matrix satisfying $\etar^\perp z_n = 0$ for all
$n=1,\ldots,\znum$. Then we can rewrite the DADVI objective as
\begin{align}
\klobj{\eta | \Z}  = -\log |\etar^{\Z} + \etar^{\perp}| -
    \meann \log \p(\etamu + \etar^\Z \z_n, \y). 
    \label{eq:full_rank_dadvi_objective}
\end{align}
Since $\sup_{\etar^{\perp}} \log |\etar^{\Z} + \etar^{\perp}| = \infty$, the
result follows.\footnote{Recall that the log determinant is the sum of the logs of the
eigenvalues of $\etar$, which can be made arbitrarily large as $\etar^{\perp}$
varies freely.}
\end{proof}
\end{theorem}

What will happen, in practice, if one tries to use DADVI in the full-rank case?
Denote the maximum a posteriori (MAP) estimate as $\hat\theta :=
\argmax_{\theta} \log \p(\theta, \y)$, and note that the first term on the right
hand side of \cref{eq:full_rank_dadvi_objective} is most negative when $\etamu =
\hat\theta$ and $\etar^\Z$ is the zero matrix.  A zero $\etar^\Z$ is
impermissible because, when $\etar^\Z$ is actually the zero matrix, then
$\etar^{\Z} + \etar^{\perp}$ is singular, and $\log |\etar^{\Z} + \etar^{\perp}|
= -\infty$.\footnote{ Indeed, if $\etar^\Z = 0$, then $\q(\theta \vert \eta)$
would have zero variance in any direction spanned by $\Z$, $\p(\theta, \y)$
would not be absolutely continuous with respect to $\q(\theta \vert \eta)$, and
the reverse KL divergence would be undefined.} However, for any $\varepsilon >
0$ and $M > 0$, we can take $\etar^\Z \z_n = \varepsilon \z_n$ and $\etar^\perp
v = M v$ for any $v \perp \Z$, so that $\etar$ is full-rank. 
When $\etamu = \hat\theta$, one can always decrease both terms on the right hand
side of \cref{eq:full_rank_dadvi_objective} via the following two-step
procedure. First, decrease $\varepsilon$ by any amount and thereby decrease the
first term. Second, given that $\varepsilon$, increase $M$ by a sufficient
amount to decrease the second term as well.

The degeneracy described in \cref{thm:fradvi} can be avoided if one uses at
least as many draws as there are model parameters, i.e., if $\znum \ge
\thetadim$.  However, based on the classical optimization results discussed
above (e.g. \citet[Chapter 5]{shapiro:2021:lectures}), one might expect
full-rank ADVI to require $\znum$ to be on the order of $\thetadim^2$, since the
full-rank variational parameters have dimension of order $\thetadim^2$ due to
the inclusion of a full-rank covariance matrix.  We proved above that the
classical dimension dependence of $\znum$ on the dimension of the variational
parameters is unnecessarily pessimistic for certain mean-field ADVI objectives.
It is an interesting question for future work to ask whether the classical
dimension dependence is also pessimistic for the full-rank approximation: that
is, whether DADVI for full-rank ADVI actually requires $\znum$ to be on the
order $\thetadim$ rather than $\thetadim^2$, or somewhere in between.

Finally, we note that the failure of DADVI in the full-rank case appears to be
indicative of a general phenomenon. Any smooth function mapping the columns of
$\Z$ into $\thetadom$ must span an $\znum$-dimensional sub-manifold of
$\thetadom$.  If a variational approximation is rich enough to increase the
entropy to an arbitrary degree on the complement of this submanifold, then DADVI
will lead to a degenerate solution. In this sense, it is in fact the
inexpressivity of the mean-field variational assumption that allows DADVI to
work in high dimensions.


\section{Related work}\label{sec:related_work}
As discussed above in \cref{sec:intro}, the idea of approximating an intractable
optimization objective $F(\eta) := \expect{\normz}{f(\eta, \z)}$ by
$\hat{F}(\eta | \Z) := \meann f(\eta, \z_n)$ is well-studied in the optimization
literature as the ``sample average approximation'' (SAA) \citep[][Chapter
5]{nemirovski:2009:sgdvsfixed,royset:2013:optimalsaa,
kim:2015:guidetosaa,shapiro:2021:lectures}. A key theoretical conclusion of the
optimization literature is that, in general, SAA should perform worse than SG in
high dimensions in terms of computational cost of providing an accurate optimum.
Our theoretical results of \cref{sec:high_dim} and experimental results of
\cref{sec:experiments} suggest that these general-case analyses may be unduly
pessimistic for many BBVI problems, though we believe more work remains to be
done establishing guarantees for SAA applied to BBVI in high dimensions.

The present work and the concurrent work by \citet{burroni:2023:saabbvi}
together form the first systematic studies of the accuracy of SAA for BBVI,
though the idea of applying SAA to BBVI has occurred several times in the
literature in the context of other methodological results
\citep{giordano:2018:covariances,domke:2018:importanceweightingvi,domke:2019:divideandcouplevi,
broderick:2020:automatic,wycoff:2022:sparsebayesianlasso,giordano:2023:bnp}. The
methods and experiments of \citet{burroni:2023:saabbvi} provide a complement to
our present work. \citet{burroni:2023:saabbvi} propose and study a method for
iteratively increasing the number of draws used for the SAA approximation until
a desired accuracy is reached (see also \citet{royset:2013:optimalsaa} for a
similar approach in the optimization literature); in contrast, we keep the
number of draws fixed in our theoretical analysis and our experiments.
Additionally, the models considered by \citet{burroni:2023:saabbvi} are
relatively low-dimensional, which allow the authors to use a very large number
of draws (up to $\znum = 2^{18}$) without incurring a prohibitive computational
cost.  In contrast, almost all of our experiments in \cref{sec:experiments} use
$\znum = \DADVINumDraws$; only in our investigation of Monte Carlo error in
\cref{sec:experiments_sampling_variability} do we examine changing $\znum$, and
there we consider only $\znum$ up to 64.  Studying relatively lower-dimensional
models with a large number of draws allows \citet{burroni:2023:saabbvi} to apply
SAA with the full-rank approximation (see our discussion of the SAA with the
full-rank approximation in \cref{sec:dadvi_full_rank}).  In contrast, we
emphasize the computation of LR covariances with the SAA approximation (as in
\citet{giordano:2018:covariances}) and on the use of DADVI in higher dimensions
more generally.  One could imagine combining our approaches: for example, by
computing the size of the sampling error relative to the LR covariance, and
increasing the number of draws as necessary as recommended in
\citet{burroni:2023:saabbvi}, though we leave such a synthesis for future work.


\section{Experiments}\label{sec:experiments}
We consider a range of models and datasets. We find that, despite using
out-of-the box optimization and convergence criteria, DADVI optimization (using
the SAA approximation) typically converges much faster than classical
(stochastic) ADVI. DADVI performs comparably to ADVI in posterior mean
estimation while allowing much better posterior covariance estimation via linear
response. Upon examination of optimization trajectories, we find that ADVI tends
to eventually find better ADVI objective values than DADVI but typically takes
longer to do so. And we confirm that the sampling variability estimates
available from DADVI are of high quality, even for just tens of draws.

Below, DADVI exhibits good performance on a number of high-dimensional models.
These models do not obviously satisfy any of the theoretical conditions for good
performance of the SAA established above (\cref{sec:high_dim}).  So our
experimental results point to a gap between theory and experiment that is an
interesting subject for future work.

In our experiments, as in the rest of the paper, we follow the convention that
``ADVI'' refers to methods that use stochastic optimization, and ``DADVI''
refers to our proposal of using SAA with the ADVI objective function.

\subsection{Models and data} \label{sec:models_data}
We evaluate DADVI and ADVI on the following models and datasets.
\begin{itemize}
    \item ARM: $\ARMNumModels$ models and datasets taken from a
    hierarchical-modeling textbook \citep{gelman:2006:arm}.  The datasets are
    relatively small and the models consist of textbook linear and generalized
    linear models, with and without random effects.
    \item Microcredit: A hierarchical model from development economics
    \citep{meager:2019:microcredit} that performs shrinkage on seven
    randomized controlled trials. The model accounts for heavy tails, asymmetric
    effects, and zero-inflated observations.
    \item Occupancy: A multi-species occupancy model from ecology
    \citep{ingram:2022:occupancy,kery:2009:speciesrichness}. In occupancy
    models, the question of interest is whether a particular species is present
    at (i.e.\ occupying) a particular site. The data consist of whether the
    species was observed at repeated visits to the site. At any given visit, the
    species may be present but not observed. Occupancy models estimate both (1)
    the suitability of a site as a function of environmental covariates such as
    temperature or rainfall and (2) the probability of observing the species
    given that it is present (the observation process). The resulting likelihood
    makes it a non-standard regression model and thus a good candidate for a
    black-box inference method. Here we use a multi-species occupancy model that
    places a hierarchical prior on the coefficients of the observation process.
    Our dataset comprises 1387 sites, 43 environmental covariates at each site,
    32 different species, and 2000 visits; this dataset represents a subset of
    the eBird dataset used by \cite{ingram:2022:occupancy}.\footnote{We used a
    subset so that our ground-truth MCMC method would complete in a reasonable
    amount of time.}
    \item Tennis: A Bradley-Terry model with random effects for ranking tennis
    players. In this model, each tennis player has a rating, assumed
    fixed throughout their career. The probability of a given player beating
    another is determined by the inverse logit of their rating difference.
    The ratings are modeled as random effects, and the data comprises all
    men's professional tennis matches on the ATP tour since 1969. Overall,
    this is a large dataset of 164,936 matches played between 5,013 different players,
    each of whom has their own random effect, making this a high-dimensional mixed model.
    \item POTUS: A time series polling model for the US presidential election
    \citep{heidemanns:2020:presidential}.  This model is both complex and
    high-dimensional. It models logit polling probabilities with a reverse
    autoregressive time series and random effects for various polling
    conditions.
\end{itemize}
Throughout this section, by a ``model'' we will mean a model with its
corresponding dataset.

\begin{table}[h!]
\begin{tabular}{|c|c|c|c|}
\hline\hline
Model Name  &   $\thetadim$     & NUTS runtime \\
\hline\hline
ARM ($\ARMNumModels$ models) &
$\ARMMinParamDim$ to $\ARMMaxParamDim$ (median $\ARMMedParamDim$)&
$\ARMMinNUTSSeconds$ seconds to $\ARMMaxNUTSMinutes$
    minutes (median $\ARMMedNUTSSeconds$ seconds) \\
\hline
Microcredit & $\MCParamDim$ & $\MCNUTSMinutes$ minutes\\
\hline
Occupancy & $\OccParamDim$ & $\OccNUTSMinutes$ minutes\\
\hline
Tennis & $\TennisParamDim$ & $\TennisNUTSMinutes$ minutes\\
\hline
POTUS & $\PotusParamDim$ & $\PotusNUTSMinutes$ minutes\\
\hline\hline
\end{tabular}
\caption{Model summaries.}
\label{tab:model_desc}
\end{table}

These models differ greatly in their complexity, as can be seen in
\cref{tab:model_desc}. The $\ARMNumModels$ ARM models from
\cite{gelman:2006:arm} are generally simple,\footnote{Indeed, many of the ARM
models can be fit quickly enough with MCMC that BBVI is arguably not necessary.
We include all the ARM models in our results to show that DADVI works well in
both simple and complex cases.} ranging from fixed effects models with a handful
of parameters to generalized linear mixed models with a few hundred parameters.
The other four models are more complex, with total parameter dimension
$\thetadim$ ranging from $\MCParamDim$ for the Microcredit model to
$\PotusParamDim$ for the POTUS model. We restricted attention to posteriors that
could be tractably sampled from with the NUTS MCMC algorithm
\citep{hoffman:2014:nuts} as implemented in PyMC \citep{salvatier:2016:pymc3} in
order to have access to ``ground truth'' posterior means and variances. However,
outside the relatively simple ARM models, NUTS samplers were time-consuming,
which motivates the use of faster variational approximations.

We fit each model using the following methods, including three different
versions of ADVI.
\begin{itemize}
\item NUTS: The ``no-U-turn'' MCMC sampler in PyMC
\citep{salvatier:2016:pymc3}. 
%
\item DADVI:  Except where otherwise indicated, we report results with
$\znum=\DADVINumDraws$ draws for DADVI for each model. We optimized using an
off-the-shelf second-order Newton trust region method (\texttt{trust-ncg} in
\texttt{scipy.optimize.minimize}).  Our implementation of DADVI is available as a
Python package at \url{https://github.com/martiningram/dadvi}.
\item LRVB: Using the optimum found by DADVI, we computed linear response
covariance estimates.  In the high-dimensional models Occupancy, Tennis, and
POTUS, we selected a small number of quantities of interest and used the
conjugate gradient (CG) algorithm to compute the LR covariances and frequentist
standard errors.  For Occupancy, the quantities of interest were predictions of
organism presence at $\OccNumCGParams$ sites; for Tennis the quantities of
interest were win predictions of $\TennisNumCGParams$ randomly chosen matchups;
and, for POTUS, the quantity of interest was the national vote share received by
the democratic candidate on election day.  When using the CG algorithm, we
preconditioned using the estimated variational covariance as described in
\cref{app:preconditioning}. When reporting metrics for the computational cost of
computing LRVB, we always report the total cost of the posterior
approximation --- i.e., the cost of DADVI optimization plus the additional cost of
computing the LR covariances.
\item Mean field ADVI (ADVI): We used the PyMC implementation of
ADVI, together with its default termination criterion. Every 100 iterations,
this termination criterion compares the current parameter vector with the one
100 iterations ago. It then computes the relative difference for each parameter
and flags convergence if it falls below $10^{-3}$. We ran ADVI for up to 100,000
iterations if convergence was not flagged before then.
\item RAABBVI (ADVI): RAABVI represents a state-of-the art stochastic
mean field ADVI method employing principled step size selection and convergence
assessment \citep{welandawe:2022:robustbbvi}.  To run RAABBVI, we used the
public package
\texttt{viabel},\footnote{\url{https://github.com/jhuggins/viabel}} provided
by \citet{welandawe:2022:robustbbvi}. By default, \texttt{viabel} supports the
packages \texttt{autograd} and \texttt{Stan}. To be able to run RAABBVI with
PyMC, we provide it with gradients of the objective function computed with
PyMC's \texttt{JAX} backend, which we use also for DADVI.
\item Full-rank ADVI (ADVI): When possible, we used the PyMC
implementation of full-rank ADVI, together with the default termination
criterion for ADVI described above.  Full-rank was computationally prohibitive
for all but the ARM and Microcredit models.
\end{itemize}
%
%

\subsection{Computational cost}\label{sec:experiments_runtime}
\RuntimeARM{}
\RuntimeNonARM{}

We first show that, despite using out-of-the box optimization and convergence
criteria, DADVI optimization typically converges faster than the ADVI
methods. DADVI also converges much more reliably; in many cases, the
ADVI methods either converged early according to their own criteria or failed to
converge and had to be terminated after a large, pre-determined number of draws.

We measured the computational cost of a method in two different ways: the wall
time (``runtime''), and the number of model gradient or Hessian-vector product
evaluations (``model evaluations'').  Neither is a complete measure of a
method's computational cost, and we hope to provide a more thorough picture by
reporting both.  For example, we were able to naively parallelize DADVI by
evaluating the model on each draw of $\Z$ in parallel, whereas ADVI uses a
single draw per gradient step and cannot be parallelized in this way. As a
consequence, DADVI will have a favorable runtime relative to ADVI for the same
number of model evaluations.

We included NUTS runtime results as a baseline.  We do not include model
evaluations for NUTS, since standard NUTS packages do not typically report the
number of model evaluations used for leapfrog steps that are not saved as part
of the MCMC output.

The results for ARM and non-ARM models are shown respectively in
\cref{fig:runtimes_arm_graph,fig:runtimes_nonarm_graph}.  Both DADVI and LRVB
are faster than all competing methods in terms of both runtime and model
evaluations on most models, with the exception of a small number of ARM models
and the Occupancy model. These computational benefits are favorable for DADVI
and LRVB given the results of \cref{sec:experiments_posterior_accuracy} below
showing that the posterior approximations provided by DADVI and LRVB are similar
to or better than the posterior approximations from the other methods.

\subsection{Posterior Accuracy}\label{sec:experiments_posterior_accuracy}
\PosteriorAccuracyARM{}
\PosteriorAccuracyNonARM{}

We next see that the quality of posterior mean estimates for DADVI and the ADVI
methods are comparable. The LRVB posterior standard deviations are much more
accurate than the ADVI methods, including full-rank ADVI. 

Each method produced a posterior mean estimate for each model parameter,
$\mu_\method$, and a posterior standard deviation estimate, $\sigma_\method$.
Above, we used $\mu$ to denote the posterior expectation of the full $\theta$
vector, but here we are using it more generically to denote a posterior
expectation of some sub-vector of $\theta$, or even the posterior mean of a
transformed parameter as estimated using Monte Carlo draws from the variational
approximation in the unconstrained space. We use the NUTS estimates, $\mu_\nuts$
and $\sigma_\nuts$, as the ground truth to which we compare the various
variational methods.  In order to form a common scale for the accuracy of the
posterior means and variances, we define the relative error in the posterior
mean and standard deviation as follows:
\begin{align*}
\muerr_\method :={} \frac{\mu_\method - \mu_\nuts}{\sigma_\nuts}
\quad\textrm{and}\quad
\sderr_\method :={} \frac{\sigma_\method - \sigma_\nuts}{\sigma_\nuts}.
\end{align*}
For example, if, on a particular parameter of a particular model, we find that
$\norm{\muerr_\dadvi} < \norm{\muerr_\sadvi}$, we would say that DADVI has
provided better mean estimates of that model parameter than mean-field ADVI. For
posterior covariances we will always report $\sderr_\lrvb$ rather than
$\sderr_\dadvi$, since we expect $\sigma_\dadvi$ to suffer from the same
deficiencies as the ADVI methods due to their shared use of the
mean-field approximation.

As discussed in \cref{sec:setup}, any parameters with restricted ranges will
typically be transformed before running ADVI. In our plots, then, we include one
point each for the original and transformed versions, respectively, of each
distinctly named parameter in the PyMC model. For Occupancy, Tennis, and POTUS,
we reported posterior mean accuracy measures for all parameters, but posterior
uncertainty measures only for a small number of quantities of interest. When a
named parameter is multi-dimensional, we report the norm of the error vector
over  all dimensions in order to avoid giving too much visual weight to a small
number of high-dimensional parameters.

The posterior accuracy results for ARM and the larger models are shown
respectively in \cref{fig:posterior_arm_graph,fig:posterior_nonarm_graph}. Recall that,
of
the non-ARM models, only the Microcredit model was small enough for full-rank
ADVI.

The estimates for the posterior means are comparable across methods, with
RAABBVI performing the best on average.  However, there are parameters for which
RAABBVI's mean estimates are off by up to a hundred standard deviations while
the DADVI estimates are fairly accurate.  In contrast, when the DADVI mean
estimates are severely incorrect, the RAABBVI ones are also severely incorrect.
This pattern suggests that severe errors in the DADVI posterior means are
primarily due to the mean-field approximation, whereas severe errors in ADVI
methods can additionally occur due to problems in optimization.

The LRVB posterior standard deviation estimates are almost uniformly better than
the ADVI and RAABBVI estimates based on the mean-field approximation. This
performance is not surprising since the mean field approximation is known to
produce poor posterior standard deviation estimates.\footnote{Note that the
relative standard deviation errors for ADVI tend to cluster around 1 because
MFVB posterior standard deviations tend to be under-estimated, and so a small
posterior standard deviation estimate leads to a relative error of one.}
Interestingly, for the ARM models, even the full-rank ADVI posterior covariance
estimates are worse than the LRVB covariance estimates, which is probably due to
the difficulty of optimizing the full-rank ADVI objective.

\subsection{Assessing convergence}\label{sec:experiments_convergence}
\TracesARM{}
\TracesNonARM{}

By examining the optimization traces, we next see that the ADVI methods
eventually find better optima (in terms of the variational objective) than
DADVI, but they typically take longer than DADVI to terminate, in agreement with
\cref{sec:experiments_runtime}.

In order to understand the progress of ADVI and RAABBVI towards their optimum,
we evaluated the variational objective on a set of 1000 independent
draws\footnote{We used the same set of independent draws for each method
to ensure a like-to-like comparison.} for each method along its optimization.
This evaluation is computationally expensive, but gives a good estimate of the
true objective $\klfullobj{\cdot}$ along the optimization paths. Specifically,
letting $\eta^{i}_{\method}$ denote the variational parameters for method
$\method$ after $i$ model evaluations, and letting $\Zindep$ denote the set of 1000
independent draws, we evaluated $\klobj{\eta^{i}_{\method} \vert \Zindep}$ for
each method and for steps $i$ up to convergence.

In order to place the optimization traces on a common scale, for each
method we center and scale the objective values by the DADVI optimum
and sampling standard deviation. In particular, we report $\kappa^i_{\method}$,
which is equal to
\begin{align}\label{eq:trace_normalization}
\kappa^i_{\method} :=
\frac{
        \klobj{\eta^{i}_{\method} \vert \Zindep} -
        \klobj{\etahat_{\dadvi} \vert \Zindep}}
    {\sqrt{\varhat{\Zindep}{\klobj{\etahat_{\dadvi} \vert \z}}}}
\end{align}
where $\varhat{\Zindep}{\klobj{\etahat_{\dadvi} \vert \z}}$ denotes an
approximation to $\var{\normz}{\klobj{\etahat_{\dadvi} \vert \z}}$ using the
sample variance over $\Zindep$.  Let $i^*_{\method}$ denote the number of model
evaluations taken by a method at convergence.  Then, under
\cref{eq:trace_normalization}, $\kappa^{i^*_{\dadvi}}_{\dadvi} = 1$ by
definition, $\kappa^{i^*_{\method}}_{\method} < 1$ indicates a better optimum at
convergence for $\method$ relative to DADVI, and $i^*_{\method} < i^*_{\dadvi}$
indicates faster convergence for $\method$ in terms of model evaluations
relative to DADVI.  The paths traced by $\kappa^i_{\method}$ may be
non-monotonic because the algorithms do not have access to $\Zindep$.

The optimization traces for ARM and non-ARM models are shown respectively in
\cref{fig:traces_arm_graph,fig:traces_nonarm_graph}, with suitably transformed
axes for easier visualization.  In many cases, the ADVI methods
eventually find better optima (in terms of the variational objective) than
DADVI, but ADVI typically takes longer to do so (the slower
convergence is also shown in
\cref{fig:runtimes_arm_graph,fig:runtimes_nonarm_graph}).  As can be seen on the
non-ARM models in \cref{fig:traces_nonarm_graph}, the ADVI methods
sometimes reach lower objective function values sooner than DADVI, but continue
to optimize because they do not have access to the computationally expensive
$\klobj{\eta^{i}_{\method} \vert \Zindep}$ and have not detected convergence
according to their own criteria.  Similarly, DADVI sometimes finds lower values
of $\klfullobj{\cdot}$ along its path to optimization, but does not terminate
because these points correspond to sub-optimal values of $\klobj{\cdot \vert
\Z}$.

The results in \cref{fig:traces_arm_graph,fig:traces_nonarm_graph} suggest the
possibility of initializing ADVI with DADVI and then optimizing further with
stochastic methods in cases when low values of the objective function are of
interest.  However, as seen in \cref{sec:experiments_posterior_accuracy} above,
lower values of the variational objective do not necessarily translate into
better posterior moment estimates.

\subsection{Sampling variability} \label{sec:experiments_sampling_variability}

\CoverageHistogram{}

We next show that frequentist standard error estimates from DADVI
provided good estimates of the sampling variability of the DADVI mean estimates, particularly
for $\znum \ge 32$.

As discussed in \cref{sec:mc_error_estimation}, the sampling variability of
DADVI estimates are straightforward to compute using standard formulas for the
sampling variability of M-estimators. For the DADVI mean estimates, we computed
the sampling standard deviation as described in
\cref{sec:mc_error_estimation,sec:lr_mc_computation}.\footnote{For the large
POTUS, Occupancy, and Tennis models, we used CG to compute frequentist coverage
for the same select quantities of interest for which we computed LR
covariances.} We denote by $\freqsd$ our estimate of
$\sqrt{\var{\normz}{\mu_\dadvi}}$ as computed using \cref{eq:mc_variance}, that
is, of the sampling standard deviation of the DADVI mean estimate under sampling
of $\Z$.  We can evaluate the accuracy of $\freqsd$ by computing $\mu_\dadvi$
with a large number of draws, which we denote as $\mu_\infty$, and checking
whether
\begin{align*}
\freqerr := \frac{\mu_\dadvi - \mu_\infty}{\freqsd}
\end{align*}
has an approximately standard normal distribution under many draws of
$\mu_\dadvi$.  We evaluated $\mu_\infty$ by taking the average of 100 runs with
$\znum = 64$ each.\footnote{The values shown in the $\znum = 64$ panel of
\cref{fig:coverage} are the same as those whose average was taken to estimate
$\mu_\infty$.  In theory, this induces some correlation between the $\freqerr$
values for $\znum = 64$.  However, the sampling variability of $\mu_\infty$ was
so small that the induced correlation is practically negligible.}

To evaluate whether $\freqerr$ has a normal distribution, we can take  $\Phi$ to
be the cumulative distribution function of the standard normal distribution, and
check whether $\Phi(\freqerr)$ has a uniform distribution.  Since the parameters
returned from a particular model are not independent under sampling from $\Z$,
the $\Phi(\freqerr)$ are not independent, and standard tests of uniformity like
the Kolmogorov-Smirnov test are not valid.
However, we can visually inspect the quality of the standard errors by checking
whether $\Phi(\freqerr)$ has an approximately uniform distribution, without
attempting to quantify how close it should be to uniform by chance alone. As can
be seen in \cref{fig:coverage}, for $\znum = 8$ and $\znum = 16$ the
$\Phi(\freqerr)$ values are over-dispersed to varying degrees for different
models; this behavior indicates that the sampling variance $\xi$ is under-estimated. In
contrast, the intervals provide good marginal coverage when $\znum \ge 32$,
though some over-dispersion remains in the Occupancy model.
%


\section{Conclusion}
In this paper, we proposed performing deterministic optimization on an approximate objective
instead of using traditional stochastic optimization 
on the intractable objective from the mean-field ADVI problem.
We found that using our DADVI approach can be faster, more accurate, and more automatic.
The benefits
of a deterministic objective can be attributed to the ability to use
off-the-shelf second-order optimization algorithms with simple convergence
criteria and linear response covariances.  Additionally, the use of a
deterministic objective allows computation of Monte Carlo sampling errors for
the resulting approximation. And these errors can facilitate an explicit tradeoff between computation
and accuracy.  In contrast to the worst-case analyses in the optimization
literature, we show theoretically that the number of samples needed for the
deterministic objective need not scale linearly in the dimension in types of
statistical models commonly encountered in practice.  Although a deterministic
objective cannot be used with highly expressive approximating families (such as
full-rank ADVI), there is reason to believe that deterministic objectives can
provide practical benefits for many black-box variational inference problems.

\section{Acknowledgements}
Ryan Giordano and Tamara Broderick were supported in part by an NSF CAREER Award
and an ONR Early Career Grant.  We are indebted to Ben Recht and Jonathan
Huggins for helpful discussions and suggestions.  We are also
grateful for the feedback from our anonymous reviewers.  All mistakes are our own.

\newpage
\bibliography{references.bib}

\newpage
\appendix

\section{Elaboration on the mean-field assumption}\label{app:mean-field}
In practice, the mean-field assumption in variational inference need not always
correspond to factorization over every single one-dimensional component of each
parameter. Rather, it often represents a factorization into individual
parameters as described in a model. For instance, consider a parameter within a
model that represents a distribution over $K$ outcomes, so that its elements are
positive and sum to one. A natural prior for such a parameter might be a
Dirichlet distribution. If this parameter exists as one parameter among multiple
parameters in our model, a mean-field assumption will typically provide a
separate factor for this parameter, but it will not further factorize across
components within the parameter. So $\Sigma(\eta)$ may, in fact, be
block-diagonal rather than purely diagonal, where each block size will
correspond to the size of a parameter. 

Researchers have explored other options between the extremes of the mean-field
and full-rank assumptions for Gaussian approximations within variational
inference; see, for instance, \citep{zhang:2022:pathfinder}.

\section{Behavior of high-dimensional normals}\label{app:high_dim_normal}

\def\zbar{\bar{z}}
\def\zzbar{\overline{zz^\trans}}
\def\sigmat{S}
\def\zcov{\hat{\Sigma}_{z}}

\subsection{Proof of \cref{prop:normal_accurate}}\label{app:normal_accurate_proof}
We begin by deriving the DADVI optimal estimates.  Let $\zbar :=
\meann \z_n$ and $\zzbar := \meann \z_n \z_n^\trans$.  Also, let $\sigmat :=
\diag{\etasigma}$, noting that $\sigmat v = \etasigma \odot v$ for any vector
$v$.
We can
write $\theta_n = \etamu + \sigmat \z_n$, so
\begin{align*}
\expecthat{\Z}{\theta(\z, \eta)} = \etamu + \sigmat \zbar
\quad\textrm{and}\quad
\expecthat{\Z}{\theta(\z, \eta) \theta(\z, \eta)^\trans} =
    \etamu \etamu^\trans +
    \etamu  \zbar^\trans \sigmat +
    \sigmat \zbar \etamu^\trans +
    \sigmat \zzbar \sigmat,
\end{align*}
so
\begin{align}\label{eq:mvn_dadvi_obj}
\klobj{\eta} ={}&
    \frac{1}{2} \etamu^\trans A \left(\etamu + 2 \sigmat \zbar \right) +
    \frac{1}{2} \trace{A \sigmat \zzbar \sigmat}
    -B^\trans (\etamu + \sigmat \zbar)
    -\sumd \log \etasigma[d].
\end{align}
For a fixed $\etasigma$ (and so a fixed $\sigmat$), the DADVI optimal
mean parameter then satisfies
\begin{align}\label{eq:normal_muhat}
A \left(\etamuhat + \sigmat \zbar \right) - B = 0
\quad\Rightarrow\quad
\etamuhat = A^{-1} B - \sigmat \zbar
= \etamu^{*} - \sigmat \zbar.
\end{align}
Thus, for any particular entry $d$, $\etamuhat[d] - \etamu[d]^{*} =
O_p(N^{-1/2})$ as long as $\etasigma[d] = O_p(1)$, both as the
number of samples, $\znum$, goes to infinity.

We can now turn to the behavior of $\etasigmahat$. By plugging $\etamuhat$ as a
function of $\sigmat$, which is given by \cref{eq:normal_muhat}, into each term
of \cref{eq:mvn_dadvi_obj} that depends on $\mu$, we get
\begin{align*}
\frac{1}{2} \etamuhat^\trans A \left(\etamuhat + 2 \sigmat \zbar \right) ={}&
\frac{1}{2} (\etamuhat + \sigmat \zbar - \sigmat \zbar)^\trans
    A \left(\etamuhat + \sigmat \zbar +\sigmat \zbar \right)
\\={}&
\frac{1}{2} (A^{-1} B - \sigmat \zbar)^\trans
    A \left(A^{-1} B +\sigmat \zbar \right)
\\={}&
\frac{1}{2} B^\trans A^{-1} B - \frac{1}{2} \zbar^\trans S A S \zbar
\quad\textrm{and}\\
B^\trans (\etamuhat + \sigmat \zbar) ={}& B^\trans A^{-1} B.
\end{align*}
Plugging the preceding two equations into the corresponding terms of
\cref{eq:mvn_dadvi_obj} gives, up to a constant $\const$ that does not depend on
$\etasigma$,
\begin{align}\label{eq:normal_klhat_sigma_profile}
\klobj{\etasigma} ={}&
\frac{1}{2} \trace{A \sigmat \left(\zzbar - \zbar \zbar^\trans\right) \sigmat}
-\sumd \log \etasigma[d] + \const.
\end{align}
Let $R$ denote the symmetric square
root of the symmetric, positive definite $A$ matrix (so $A = RR$ and $R =
R^\trans$).  Then we have
\begin{align*}
\trace{A S (\zzbar - \zbar \zbar^\trans) S} =
    \trace{R S
        \left(\zzbar - \zbar \zbar^\trans \right) (R S)^\trans}.
\end{align*}

Let $\dequal$ denote equality in distribution, i.e., $X \dequal Y$ means that
$X$ and $Y$ have the same law.  Then
\begin{align*}
R S \z_n \dequal (RSSR)^{1/2} z_n,
\end{align*}
since both the left and the right hand sides of the preceding display have a
$\normal{\cdot}{\zerod[\thetadim], RSSR}$ distribution.  (We have used the fact that
$S$ and $R$ are both symmetric.)  Thus, for any $\etasigma$,
\begin{align}\label{eq:klobj_normal_dequal}
\klobj{\etasigma} \dequal
\frac{1}{2} \trace{R \sigmat \sigmat R \left(\zzbar - \zbar \zbar^\trans\right)}
-\frac{1}{2}\sumd \log \etasigma[d]^2 + \const.
\end{align}
Though the dependence on $\Z$ of the left and right hand sides of the preceding
equation is different, for a given $\etasigma$, the two have the same
distribution, and their optima have the same distribution as well.
The product $\sigmat \sigmat$ is simply $\diag{\etasigma^2}$, so
expanding the trace gives
\begin{align*}
\trace{R \sigmat \sigmat R \left(\zzbar - \zbar \zbar^\trans\right)}
={}&
\sum_{i,j,k=1}^{\thetadim} R_{ij} \etasigma[j]^2 R_{jk}
 \left(\zzbar - \zbar \zbar^\trans\right)_{ki} \Rightarrow \\
\frac{\partial}{\partial \etasigma[d]^2}
    \trace{R \sigmat \sigmat R \left(\zzbar - \zbar \zbar^\trans\right)}
    ={}&
\sum_{i,k=1}^{\thetadim} R_{dk}
     \left(\zzbar - \zbar \zbar^\trans\right)_{ki} R_{id}
    \\ ={}&
     (R \left(\zzbar - \zbar \zbar^\trans\right) R^\trans)_{dd}.
\end{align*}
So the optimal value of $\etasigma[d]^2$ for the right hand side of
\cref{eq:klobj_normal_dequal} is
\begin{align*}
\etasigmahat[d]^2 =
\frac{1}{(R \left(\zzbar - \zbar \zbar^\trans\right) R^\trans)_{dd}}.
\end{align*}
Note that $R \z_n \sim \normal{\cdot}{\zerod[\thetadim], A}$.  Therefore,
if $w_n \sim \normal{\cdot}{\zerod[\thetadim], A}$, then
\begin{align*}
\etasigmahat[d]^{-2}
\dequal
\meann w_{nd}^2 - \left(\meann w_{nd}\right)^2.
\end{align*}
So $\expect{\normz}{\etasigmahat[d]^{-2}} = \frac{N - 1}{N} A_{dd} =
\frac{N - 1}{N} (\etasigma[d]^{*})^{-2}$, and $\etasigmahat[d]^{-2} -
(\etasigma[d]^{*})^{-2}= O_p(N^{-1/2})$.  From this it follows that
$\etamuhat[d] - \etamu[d]^* = O_p(N^{-1/2})$ as well.

\subsection{Proof of \cref{prop:mvn_lr}}\label{app:mvn_lr_proof}
Recall that the linear response covariance estimate for $\theta$ in this model
considers the perturbed model
\begin{align*}
\log \p(\theta, \y | t) := \log \p(\theta, \y) + t^\trans \theta
\end{align*}
and computes
\begin{align*}
\lrcov{\q(\theta \vert \etahat)}{\theta} =
\fracat{d \etamuhat}{d t^\trans}{\etahat} = A^{-1},
\end{align*}
where the final equality follows from
\cref{eq:normal_muhat,eq:normal_klhat_sigma_profile} by identifying $B$ with $B
+ t$ and observing that $\etasigmahat$ does not depend on $t$.  Since $A^{-1}$
is in fact the true posterior variance, the linear response covariance is exact
in this case irrespective of how small $\znum$ is, in contrast even to
$\etasigma^{*}$, which can be a poor estimate of the marginal variances unless
$A$ is diagonal.

\section{High-dimensional global--local problems}\label{app:high_dim_global_local}


\begin{proof} of \cref{thm:global_local}.

We can write
\begin{align}
\MoveEqLeft
\klfullobj{\etahat} - \klfullobj{\etastar}  ={}
\nonumber\\&
\klfullobj{\etahat} - \klfullobj{\etastar}   
+ \klobj{\etastar | \Z} - \klobj{\etastar | \Z} + 
\klobj{\etahat | \Z} - \klobj{\etahat | \Z} = 
\nonumber\\&
\left(\klfullobj{\etahat} - \klobj{\etahat | \Z} \right) +
\left(\klobj{\etastar | \Z}  - \klfullobj{\etastar} \right) +
\left(\klobj{\etahat | \Z} - \klobj{\etastar | \Z}\right) \le
\nonumber\\&
\left(\klfullobj{\etahat} - \klobj{\etahat | \Z} \right) +
\left(\klobj{\etastar | \Z}  - \klfullobj{\etastar} \right) \le
\nonumber\\&
\abs{\klfullobj{\etahat} - \klobj{\etahat | \Z}} +
\abs{\klobj{\etastar | \Z}  - \klfullobj{\etastar}} \le
\nonumber\\&
2 \sup_{\eta \in \etadom} \abs{\klfullobj{\eta} - \klobj{\eta | \Z}}
\label{eq:kl_difference_decomposition}
\end{align}
where the penultimate inequality uses the fact that $\klobj{\etahat | \Z} -
\klobj{\etastar | \Z} \le 0$.
By \cref{assu:local_minimum}, we then have
\begin{align}
\norm{\etahat^\gamma - \etastar^\gamma}_2^2 \le
\frac{2}{P C_3} \sup_{\eta \in \etadom}
    \abs{ \klfullobj{\eta} - \klobj{\eta | \Z}}.
    \label{eq:etagamma_strict_min}
\end{align}

Similarly, for any given $p$, apply \cref{assu:local_minimum} with the
components of $\eta$ matching $\etahat^p$ in the components corresponding
to the variational distribution for $\lambda^p$, and matching $\etastar$
otherwise, giving
\begin{align}\label{eq:etap_strict_min}
\fbar^p(\etastar^\gamma, \etahat^p) -
    \fbar^p(\etastar^\gamma, \etastar^p)
    \ge C_3 \norm{\etahat^p - \etastar^p}_2^2.
\end{align}
Since $\etahat^p$ minimizes $\eta^p \mapsto \fhat^p(\etastar^\gamma, \eta^p)$,
the same reasoning as \cref{eq:kl_difference_decomposition} implies that
\begin{align*}
\fbar^p(\etastar^\gamma, \etahat^p) - \fbar^p(\etastar^\gamma, \etastar^p)
\le 2 \sup_{\eta^p \in \etadom}
    \abs{\fbar^p(\etastar^\gamma, \eta^p) -
         \fhat^p(\etastar^\gamma, \Z^\gamma, \eta^p, \Z^p).
}
\end{align*}
Combining the previous two displays gives
\begin{align*}
\norm{\etahat^p - \etastar^p}_2^2 \le
\frac{2}{C_3} \sup_{\eta \in \etadom}
    \abs{ \fbar^p(\eta^\gamma, \eta^p) -
          \fhat^p(\eta^\gamma, \Z^\gamma, \eta^p, \Z^p)}.
\end{align*}

Next, we use \cref{assu:local_ulln} to control the  difference between the
samples and limiting objectives. Take $\delta' = C_3 \delta / 2$.
Let
\begin{align*}
\err^p :={}& \sup_{\eta^\gamma, \eta^p} \abs{\fhat^p(\eta^\gamma,
\Z^\gamma, \eta^p, \Z^p) - \f^p(\eta^\gamma, \eta^p)}
\quad\textrm{and}\quad
\err :={}
\frac{1}{P} \sup_{\eta \in \etadom} \abs{ \klobj{\eta | \Z} -  \klfullobj{\eta}}.
\end{align*}
Since we can only increase the error by allowing the global parameter to vary
separately for each local ULLN, we have $\err \le \frac{1}{P} \sump \err^p$.
Therefore, $\{\forall p: \err^p \le \delta'\} \Rightarrow \{ \err \le \delta' \}$
and $\{ \err > \delta' \} \Rightarrow \{\exists p: \err^p > \delta' \}$.
A union bound then gives
\begin{align}\label{eq:ulln_error_small}
\p(\err > \delta) \le
\p\left(\bigcup_p \left\{ \err^p > \delta' \right\} \right) \le \sump
\p\left(\err^p > \delta' \right)
\le C_1 \exp\left(-C_2 N + \log P \right) \le \varepsilon,
\end{align}
where the final inequality follows from taking $N \ge N_0$ large enough to
satisfy \cref{assu:local_ulln} and
$N_0 \ge C_2^{-1} \left(
\log P - \log \left(C_1^{-1} \varepsilon\right)
\right)$.

By \cref{eq:etap_strict_min,eq:etagamma_strict_min},
\begin{align*}
\bigcap_{p=1}^P \left\{ \err^p < \delta' \right\}
\Rightarrow \bigcap_{p=1}^P \left\{\norm{\etahat^p - \etastar^p}_2^2 \le \delta
\right\}
\textrm{ and }
\bigcap_{p=1}^P
\left\{ \err^p < \delta' \right\} \Rightarrow \err < \delta' \Rightarrow
\norm{\etahat^\gamma - \etastar^\gamma}_2^2 \le \delta.
\end{align*}
The conclusion then
follows from \cref{eq:ulln_error_small}.
\end{proof}

\begin{proof} of \cref{ex:local_minimum}.

Suppose that, for each $p$, $\fbar^p(\eta^\gamma, \eta^p)$ is
twice-differentiable and convex, and the domain is compact.  Let the first and
second-order derivatives be denoted by $\nabla \fbar^p$ and $\nabla^2 \fbar^p$
respectively, and let $C_3$ lower bound the minimum eigenvalue of all
$\nabla^2 \fbar^p$.

Then a Taylor series expansion with integral remainder gives
\begin{align*}
\fbar^p(\eta^\gamma, \eta^p) - \fbar^p(\etastar^\gamma, \etastar^p) ={}&
\nabla \fbar^p(\etastar^\gamma, \etastar^p)
\begin{pmatrix}
    \eta^\gamma - \etastar^\gamma \\
    \eta^p - \etastar^p
\end{pmatrix}
+ R^p(\etastar, \eta)
\end{align*}
where
\begin{align*}
R^p(\etastar, \eta) =
\int_0^1
\begin{pmatrix}
    \eta^\gamma - \etastar^\gamma \\
    \eta^p - \etastar^p
\end{pmatrix}^\trans
\nabla^2 \fbar(\etastar^\gamma + t (\eta^\gamma - \etastar^\gamma),
               \etastar^p + t (\eta^p - \etastar^p))
\begin{pmatrix}
    \eta^\gamma - \etastar^\gamma \\
    \eta^p - \etastar^p
\end{pmatrix}
(1-t) dt.
\end{align*}
(Apply \citet[Theorem B.2]{dudley:2018:real} with $t \mapsto
\fbar(\etastar^\gamma + t (\eta^\gamma - \etastar^\gamma), \etastar^p + t
(\eta^p - \etastar^p))$.)  Since $\etastar$ is an optimum,
$\sump \nabla \fbar^p(\etastar^\gamma, \etastar^p) = 0$.  Since
$\nabla^2 \fbar^p$ is positive definite
for every $p$, there exists a $C_3 \ge 0$ such that
\begin{align*}
R^p(\etastar, \eta) \ge
    C_3 \left(\norm{\eta^\gamma - \etastar^\gamma}_2^2 +
              \norm{\eta^p - \etastar^p}_2^2 \right).
\end{align*}
It follows that
\begin{align*}
\klfullobj{\eta} - \klfullobj{\etastar} \ge
    C_3 \left( P  \norm{\eta^\gamma - \etastar^\gamma}_2^2 +
    \sump \norm{\eta^p - \etastar^p}_2^2
    \right),
\end{align*}
from which \cref{assu:local_minimum} follows.

\end{proof}

\section{Model details}\label{app:model-details}
\subsection{ARM models}

We selected $\ARMNumModels$ from the \texttt{Stan} example models
repository\footnote{\url{https://github.com/stan-dev/example-models/tree/master/ARM}}.
The models we used are as follows, with their parameter dimension in parentheses:
%
\newenvironment{simplechar}{%
   \catcode`\_=12
}{}


\texttt{\ArmModels{}}.

Some models were eliminated from consideration for being duplicates of other
models, and a small number were eliminated for poor NUTS performance (low
effective sample size or poor R hat).

\subsection{Tennis}

In the tennis model, each player, $i=1,...,M$ has a rating $\theta_i$. These
ratings are drawn from a prior distribution with a shared variance:

\begin{align}
    \theta_i \stackrel{iid}{\sim} \mathcal{N}(0, \sigma^2),
\end{align}

The standard deviation $\sigma$ is given a half-Normal prior with a scale
parameter of 1. The likelihood for a match $n=1,...,N$ between player $i$ and
$j$ is given by:

\begin{align}
    y_n  \sim \text{Bernoulli}(\text{logit}^{-1}(\theta_i - \theta_j))
\end{align}

where $y_n = 1$ if player $i$ won, and $y_n = 0$ if not.

\subsection{Occupancy model}

In occupancy models, we are interested in whether site $i$ is occupied by
species $j$. We model occupation as a binary latent variable
$y_{ij}$, with probability $\Psi_{ij}$ being the probability that the species is
occupying the site. The logit of this probability is modeled as a linear
function of environmental covariates, such as rainfall and temperature:

\begin{align}
y_{ij} \sim \textrm{Bern}(\Psi_{ij}), \\
  \textrm{logit}(\Psi_{ij}) = \bm{x}^{\text{(env)}\intercal}_i \bm{\beta}^{\text{(env)}}_j + \gamma_j, \\
  \bm{\beta}^{\text{(env)}}_j \stackrel{iid}{\sim} \mathcal{N}(0, I), \\
  \gamma_j \stackrel{iid}{\sim} \mathcal{N}(0, 10^2).
\end{align}

However, $y_{ij}$ is assumed not to be observed directly. Instead, we observe
the binary outcome $s_{ijk}$, which equals one if species $j$ was observed at
site $i$ on the $k$-th visit. If the species was observed, we know that it is
present ($y_{ij} = 1$), assuming there are no false positives. If it was not, it
may have been missed, and we model the probability that it would have been
observed if it had been present, $p_{ijk}$. Mathematically speaking, these
assumptions result in the following model:

\begin{align}
  & p(s_{ijk} = 1 \mid y_{ij} = 1) = p_{ijk}, \label{eq:det-prob} \\
  & p(s_{ijk} = 1 \mid y_{ij} = 0) = 0, \label{eq:false-pos} \\
  & \textrm{logit}(p_{ijk}) = \bm{x}_{ik}^{\text{(obs)}\intercal} \bm{\beta}^{\text{(obs)}}_j \label{eq:det-prob-model},
\end{align}

where $\bm{x}_{ik}^{\text{(obs)}\intercal}$ are a set of covariates assumed to
be related to the probability of observing the species, and
$\bm{\beta}^{\text{(obs)}}_j$ are coefficients of a linear model relating these
to the logit of the probability $p_{ijk}$.

As $y_{ij}$ is not observed, it has to be marginalized out for ADVI models to be
applicable. The resulting likelihood is given by:
\begin{align}
  p(s \mid \theta) = \prod_{i=1}^{N} \prod_{j=1}^{J} \left[(1 - \Psi_{ij}) \prod_{k=1}^{K_i} (1 - s_{ijk}) +
  \Psi_{ij} \prod_{k=1}^{K_i} (p_{ijk})^{s_{ijk}} (1 - p_{ijk})^{1 - s_{ijk}} \right].
  \label{eq:observed-data-likelihood}
\end{align}
Its derivation can be found in the appendix of \cite{ingram:2022:occupancy}.
Here, $K_i$ are the number of visits to site $i$, $N$ is the total number of
sites, $J$ is the total number of species, and the rest of the variables are as
defined previously.

\section{Preconditioning DADVI}\label{app:preconditioning}

As described in \cref{sec:lr_mc_computation}, in high-dimensional problems it is
useful to use the conjugate gradient (CG) algorithm to compute both LR
covariances and frequentist standard errors.  The CG algorithm uses products of
the form $\h v$ to approximately solve $\h^{-1}v$, and can be made more
efficient with a preconditioning matrix $M$ with $M \approx \h^{-1}$
\citep[Chapter 5]{wright:1999:optimization}.

The DADVI approximation itself provides an approximation to the upper left
quadrant of $\h^{-1}$, which can be used as a preconditioner. By the LR
covariance formula \cref{eq:lr_def},
\begin{align*}
\cov{\post(\theta)}{\theta} \approx
\lrcov{\q(\theta \vert \etahat)}{\theta} =
\begin{pmatrix}
\ident & \zerod[\thetadim \times \thetadim]
\end{pmatrix}
\h^{-1}
\begin{pmatrix}
\ident \\ \zerod[\thetadim \times \thetadim],
\end{pmatrix},
\end{align*}
which is just the upper-left quadrant of $\h^{-1}$.  Prior
to computing the LR covariances, the best available approximation
of $\cov{\post(\theta)}{\theta}$ --- and, in turn, the upper-left quadrant
of $\h^{-1}$ --- is the mean-field covariance estimate
$\cov{\q(\theta \vert \etahat)}{\theta} = \diag{\exp(\etaxihat[1]), \ldots, \exp(\etaxihat[\thetadim])}$.
Therefore, whenever using CG on a DADVI optimum, we pre-condition
with the matrix
\begin{align*}
\begin{pmatrix}
\cov{\q(\theta \vert \etahat)}{\theta} &
    \zerod[\thetadim \times \thetadim] \\
\zerod[\thetadim \times \thetadim] &
    \ident
\end{pmatrix}.
\end{align*}
Using the preceding preconditioner is formally similar to re-parameterizing the
mean parameters into their natural parameters, as when taking a natural gradient
in stochastic optimization \citep{hoffman:2013:stochasticvi}.

\end{document}